\documentclass[12pt]{article}

\usepackage[left=1in,right=1in,top=1in,bottom=1in]{geometry}
\usepackage{authblk} 
\usepackage{hyperref} 

\usepackage{amssymb,amsmath,makecell,epsfig,listings}
\usepackage{algorithm}
\usepackage{algpseudocode}
\usepackage{multirow}
\usepackage{makecell}
\usepackage{graphicx}
\usepackage{caption}
\usepackage{subcaption}
\usepackage{adjustbox}
\usepackage{listings}
\usepackage[framemethod=tikz]{mdframed}
\usepackage{booktabs}
\usepackage{listings}  
\usepackage{fancyvrb}  
\usepackage[T1]{fontenc}
\usepackage{listings}
\RequirePackage{natbib}
\newtheorem{theorem}{Theorem}
\newenvironment{proof}[1][Proof]{\noindent\textit{#1.} }{\ \rule{0.5em}{0.5em}}

\DeclareMathOperator*{\argmin}{arg\,min}

\title{Natural Language-Based Synthetic Data Generation for Cluster Analysis}
\author[1]{Michael J. Zellinger\thanks{Corresponding author: \texttt{zellinger@caltech.edu}}}
\author[2]{Peter Bühlmann}
\affil[1]{Department of Computing \& Mathematical Sciences, California Institute of Technology, Pasadena, CA 91125, USA}
\affil[2]{Seminar for Statistics, ETH Zürich, Zürich, Switzerland}

\date{}

\begin{document}

\maketitle

\begin{abstract}
Cluster analysis relies on effective benchmarks for evaluating and comparing different algorithms. Simulation studies on synthetic data are popular because important features of the data sets, such as the overlap between clusters, or the variation in cluster shapes, can be effectively varied. Unfortunately, creating evaluation scenarios is often laborious, as practitioners must translate higher-level scenario descriptions like ``clusters with very different shapes'' into lower-level geometric parameters such as cluster centers, covariance matrices, etc. To make benchmarks more convenient and informative, we propose synthetic data generation based on direct specification of high-level scenarios, either through verbal descriptions or high-level geometric parameters. Our open-source Python package \href{https://repliclust.org}{\texttt{repliclust}} implements this workflow, making it easy to set up interpretable and reproducible benchmarks for cluster analysis. A demo of data generation from verbal inputs is available at \href{https://demo.repliclust.org}{\texttt{demo.repliclust.org}}.
\end{abstract}


\section{Introduction}
\label{sec:intro}

The goal of clustering is to separate data points into groups such that points within a group a more similar to each other than to those outside the group (\citealp{McCormackClassificationReview}). In practice, it is often not clear what constitutes a cluster (\citealp{HennigTrueClusters}). As a result, many practitioners evaluate cluster analysis algorithms on synthetic data (\citealp{Milligan1980}; \citealp{Milligan1983}; \citealp{gapstat}; \citealp{Steinley2008}; \citealp{Steinley2011}; \citealp{WhitePaper}). 

Synthetic data is valuable for two reasons. First, it clearly stipulates which data points belong to which cluster, allowing objective evaluation. Second, it allows independently manipulating different aspects of the data (such as the overlap between clusters or the variability of cluster shapes), which is critical for drawing scientific conclusions about the relative merits of different cluster analysis techniques (\citealp{Milligan1996}).

Unfortunately, setting up benchmarks with synthetic data can be laborious. The process typically involves creating data sets for a number of different scenarios. For example, on benchmarks with convex clusters drawn from probabilistic mixture models, the scenarios may involve ``clusters of very different shapes and sizes'', ``highly overlapping oblong clusters'', ``high-dimensional spherical clusters'', etc. (\citealp{gapstat}). 

Existing data generators do not cater directly to such high-level scenarios. Instead, the user must carefully tune simulation parameters to arrive at the desired scenarios (\citealp{oclus}; \citealp{elki}; \citealp{mdcgen}). While some generators make it easy to control the overlaps between clusters, such high-level control typically does not extend to other aspects like the desired diversity of cluster shapes and sizes (\citealp{qiujoe06}; \citealp{hawks}).

In this paper, we explore generating synthetic data directly from high level descriptions. Our Python package, \texttt{repliclust}, accomplishes this goal by summarizing the overall geometry of probabilistic mixture models with a few high-level parameters. We use a large language model to map a user's verbal description of a scenario onto these parameters. Although our approach is based on ellipsoidal clusters, we have implemented two post-processing functions for generating more irregular cluster shapes. The first makes clusters non-convex by passing them through a randomly initialized neural network. The second makes a $p$-dimensional data set \textit{directional} by wrapping it around the $(p+1)$-dimensional sphere through an inverse stereographic projection.

\section{Generating Data from High-Level Archetypes}
\label{sec:overview}

\begin{figure}[h]
\centering
    \includegraphics[width=\textwidth]{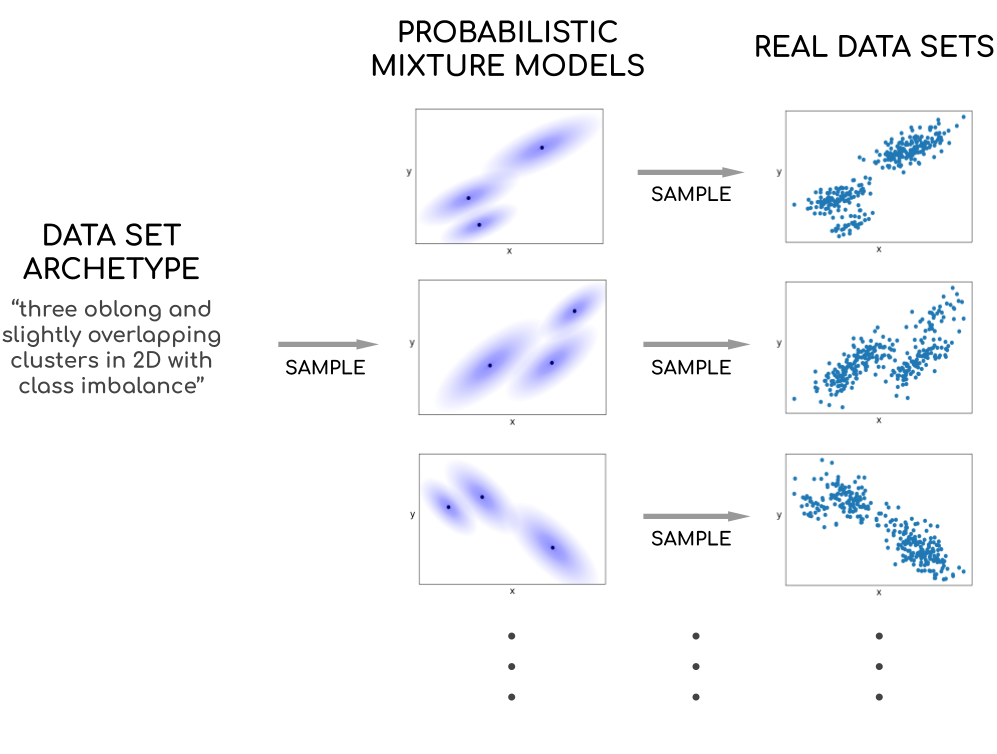}
    \caption{Illustration of synthetic data generation with data set archetypes. Left: the user specifies the desired archetype. The user can verbally describe the archetype in English or directly specify a few high-level geometric parameters. Middle: the archetype provides a random sampler for probabilistic mixture models with the desired geometric characteristics. Right: drawing i.i.d. samples from each mixture model yields synthetic data sets.}
    \label{fig:overview}
\end{figure}

Our data generator \texttt{repliclust} is based on data set archetypes. A \textit{data set archetype} is a high-level description of the overall geometry of a data set with clusters. For example, the class of all data sets with ``three oblong and slightly overlapping clusters in two dimensions with some class imbalance'' is a data set archetype.

We implement archetype-based generation by summarizing the overall geometry of probabilistic mixture models in terms of a few high-level parameters, as listed in Table \ref{tbl:archetype_params}. To create an archetype, users can either specify these parameters directly or verbally describe the archetype in English. To map an English description to a precise parameter setting, we use few-shot prompting of a large language model (\citealp{GPT3Paper}; \citealp{GPT4Report}).

\begin{table}[h!]
	\caption{Summary of high-level geometric parameters defining an \texttt{Archetype} in \texttt{repliclust}.}
  	\label{tbl:archetype_params}
  	\centering
  	\centerline{
  	\resizebox{\textwidth}{!}{
    \begin{tabular}{ | c | c |}
      \hline
      \textbf{Parameter(s)} & \textbf{Purpose} \\
      \hline
      \makecell{\texttt{n\_clusters} / \texttt{dim} / \texttt{n\_samples}} & \makecell{select number of clusters / dimensions / data points} \\
      \hline
      \makecell{\texttt{aspect\_ref} / \texttt{aspect\_maxmin}} & \makecell{determine how elongated vs spherical clusters \\ are / how much this varies between clusters} \\
      \hline
      \makecell{\texttt{radius\_maxmin}} & \makecell{determine the variation in cluster volumes} \\
      \hline
      \makecell{\texttt{max\_overlap} / \texttt{min\_overlap}} & \makecell{set maximum / minimum overlaps between clusters} \\
      \hline
      \makecell{\texttt{imbalance\_ratio}} & \makecell{make some clusters have more data points than others} \\
      \hline
      \makecell{\texttt{distributions} / \\ \texttt{distribution\_proportions}} & \makecell{select probability distributions appearing in each \\ data set / how many clusters have each distribution\footnote{Specifically, the proportion of clusters having a given distribution is the same in each data set. We are interested in adding a spread parameter that controls the variability of the distribution mix between data sets.}} \\
      \hline
    \end{tabular}
    }
    }
\end{table}

Most of the high-level geometric parameters describing an archetype are based on what we call ``max-min sampling.'' In this approach, the user controls a geometric attribute by specifying a \textit{reference value} and \textit{max-min ratio}. In addition, a constraint ensures that the reference value and is indeed typical for \textit{every} data set. For example, the aspect ratio of a cluster measures how elongated it is. The reference value \texttt{aspect\_ref} sets the typical aspect ratio among all clusters in a data set, while \texttt{aspect\_maxmin} sets the ratio of the highest to the lowest aspect ratio. To make sure that \texttt{aspect\_ref} is indeed the typical aspect ratio, max-min sampling enforces the location constraint $(\prod_{j=1}^k \alpha_j)^\frac{1}{k} = \texttt{aspect\_ref}$. Appendix A gives more details on how we manage different geometric attributes using max-min sampling.

Once an archetype has been defined, sampling concrete data sets proceeds in two steps. First, the algorithm samples a new probabilistic mixture model whose geometric structure matches the archetype. Second, we draw i.i.d. samples from this mixture model to generate a data set. Figure \ref{fig:overview} illustrates this flow.

To accommodate use cases in which variation of an archetype's hyperparameters (\texttt{n\_clusters}, \texttt{dim}, \texttt{n\_samples}) is desired , we have implemented a function \texttt{Archetype.sample\_hyperparams} for generating a list of archetypes with hyperparameters sampled from Poisson distributions centered on the original values (subject to rejection sampling based on user-specified minimum and maximum values).

\section{Sampling Probabilistic Mixture Models}
\label{sec:models}

Generating a synthetic dataset with \texttt{repliclust} starts with sampling a probabilistic mixture model that matches the desired archetype.

Sampling a mixture model proceeds through the following steps:
\begin{enumerate}
    \item Draw random cluster covariance matrices based on the archetype.
    \item Randomly initialize the cluster centers. Adjust their placement using stochastic gradient descent to meet desired constraints on the overlaps between clusters.
    \item Sample the number $n_j$ of data points per cluster, based on the extent of class imbalance specified by the archetype.
    \item Construct a data set $X$
    and cluster labels $y$ by sampling $n_j$ data points i.i.d. from the mixture component describing cluster $j$. Return $(X, y, \mathcal{A})$, where $\mathcal{A}$ is the archetype.
    \item Optionally, make cluster shapes non-convex by either \begin{enumerate}
                \item passing $X$ through a randomly
                initialized neural network (\texttt{repliclust.distort})
                \item wrapping $X \in \mathbb{R}^{n \times p}$ around the $(p+1)$-dimensional sphere to create directional data (\texttt{repliclust.wrap\_around\_sphere}).
            \end{enumerate}
\end{enumerate}

In the following sections, we give more details on each step involved in data generation.

\subsection{Defining Clusters and Mixture Models}

In the first four steps of data generation, \texttt{repliclust} models clusters as ellipsoidal probability distributions characterized by a central point. Specifically, each cluster $\mathcal{C}$ is defined by a cluster center $\mathbf{x}_\mathcal{C} \in \mathbb{R}^{p}$, orthogonal principal axes $\hat{\boldsymbol{u}}^{(1)}_\mathcal{C}, \hat{\boldsymbol{u}}^{(2)}_\mathcal{C}, ..., \hat{\boldsymbol{u}}^{(p)}_\mathcal{C}$ pointing in arbitrary directions, the lengths $\sigma^{(1)}_\mathcal{C}, \sigma^{(2)}_\mathcal{C}, ..., \sigma^{(p)}_\mathcal{C}$ of the principal axes, and a univariate probability distribution $f_\mathcal{C}(\cdot)$.

To generate an i.i.d. sample from a cluster, we 1) sample the direction $\hat{\mathbf{x}}$ from the ellipsoid defined by the cluster's principal axes, and 2) sample the length $||\mathbf{x}||_2$ according to the cluster's univariate distribution $f_\mathcal{C}$ (which can be one of many supported distributions including normal, lognormal, exponential, Student's t, gamma, chi-square, Weibull, Gumbel, F, Pareto, beta, and uniform). To make the spread of each cluster depend only on the lengths of the principal axes, we normalize each univariate distribution so that the $68.2\%$ quantiles of its absolute value is unity. For example, if the univariate distribution is exponential with rate $\lambda$, we would actually sample from a re-scaled random variable $\text{Exp}(\lambda) /q_{0.682}$, where the quantile $q_{0.682}$ satisfies $\mathbb{P}(|\text{Exp}(\lambda)| \leq q_{0.682}) = 0.682$. This rescaling puts all distributions on the same scale as the multivariate normal distribution, which natively satisfies $\mathbb{P}(|\mathcal{N}(0,1)| \leq 1) \approx 0.682$.

Figure \ref{fig:clusters}(a) visualizes clusters with different base distributions. Note that using heavy-tailed distributions can lead to far outliers (not shown). By contrast, univariate distributions with bounded support give rise to clusters with crisply defined boundaries. Stretches where a probability density function vanishes give rise to concentric holes.

\begin{figure}[h!]
    \begin{subfigure}{0.40\textwidth}
        \centering
        \includegraphics[width=\textwidth]{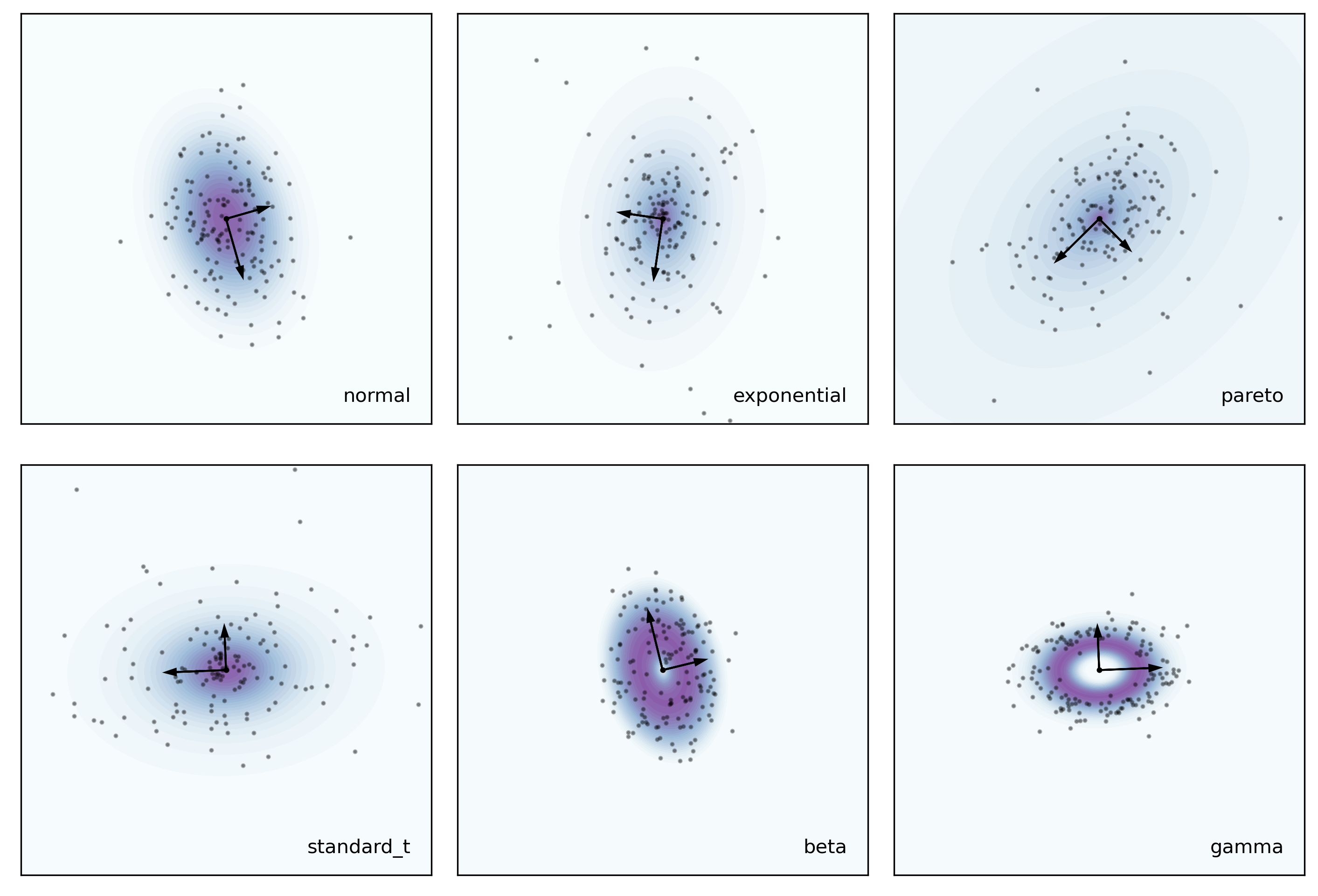}
        \caption{ }
    \end{subfigure}
    \hfill
    \begin{subfigure}{0.55\textwidth}
        \centering
        \includegraphics[width=\textwidth]{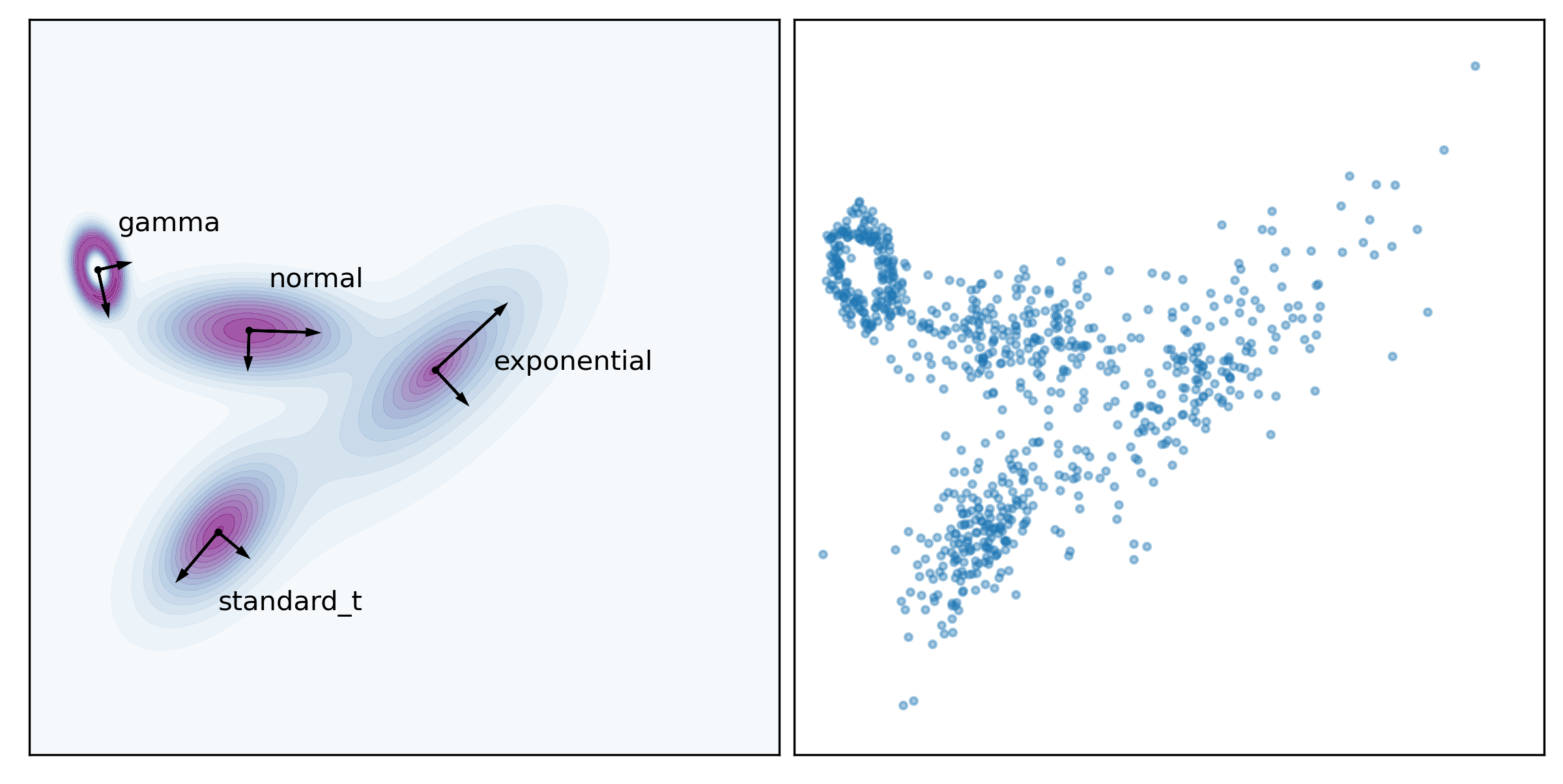}
        \caption{ }
    \end{subfigure}
    \caption{ Individual clusters (a) and a probabilistic mixture model (b) in \texttt{repliclust}. Black arrows show each cluster's principal axes. The scatter plot on the right in (b) shows a data set sampled from the mixture model. In this example, all clusters are natively 2D.}
    \label{fig:clusters}
\end{figure}

A mixture model (\texttt{repliclust.MixtureModel}) represents several clusters. Figure \ref{fig:clusters}(b) shows a two-dimensional mixture model with four clusters, and a data set sampled from it. Note that our mixture models do not model class probabilities. Instead, the number of data points per cluster are drawn using max-min sampling, based on the archetype's imbalance ratio (desired ratio between the number of data points in the most and least populous clusters). Appendix B lists the formal attributes of a mixture model.

\subsection{Managing Cluster Overlaps}
\label{sec:overlap}

Managing the degree of overlaps between clusters is one of the most important tasks of a cluster generator. In $\texttt{repliclust}$, we quantify pairwise overlap between two clusters as the irreducible error rate when classifying a new data point as belonging to one of the two clusters (assuming equal class probabilities). On the level of entire datasets, the user controls overlap by specifying the \textit{maximum} and \textit{minimum} pairwise overlap.

The maximum overlap $\texttt{max\_overlap}$ states that no pair of clusters may overlap more than a certain amount. By contrast, the minimum overlap $\texttt{min\_overlap}$ prevents isolated clusters by enforcing that each cluster has \textit{some} neighbor with which it overlaps \textit{at least} $\texttt{min\_overlap}$.
We construct a loss function that enforces the minimum and maximum overlap conditions and optimize it using stochastic gradient descent (SGD).

In this section, we first describe our notion of pairwise overlap. Second, we explain the loss function we use to enforce the maximum and minimum overlap for an entire dataset.

\subsubsection{Measuring Pairwise Overlap}

Viewing clusters as probability distributions, we formally define the overlap between two clusters as \textit{twice the minimax error rate when classifying new data points using a linear decision boundary between the two clusters}. Figure \ref{fig:overlap} illustrates this definition. To explain what we mean by ``minimax'' in this context, observe that any linear classifier $\hat{y}: \mathbb{R}^p \mapsto \{1,2\}$ depends on an axis $\boldsymbol{a} \in \mathbb{R}^{p}$ and threshold $c \in \mathbb{R}$ such that
\begin{equation}
\label{eq:defclassifier}
\hat{y}(\mathbf{x}) = \begin{cases}
			             1, & \text{if $\boldsymbol{a}^{\top} \mathbf{x} \leq c$}\\
                          2, & \text{if $\boldsymbol{a}^{\top} \mathbf{x} > c$}.
		             \end{cases}
\end{equation}
By definition, the \textit{minimax} classifier $\hat{y}^{*}$ minimizes the worst-case loss. In symbols,
\begin{equation}
\label{eq:defminimax}
    \max_{y} ~\mathbb{P}(y \neq \hat{y}^{*}(\mathbf{x}) | y) = \min_{\hat{y}} \max_{y} ~\mathbb{P}(y \neq \hat{y}(\mathbf{x}) | y),
\end{equation}
where $y \in \{1,2\}$ is the true cluster label corresponding to a new data point $\mathbf{x} \in \mathbb{R}^p$. The outer minimum on the right hand side ranges over all linear classifiers $\hat{y}$, including $\hat{y}^{*}$. Rewriting (\ref{eq:defminimax}) in terms of the classification axes $\boldsymbol{a}$ and thresholds $c$ yields
\begin{equation}
\label{eq:minimax}
    \boldsymbol{a}^{*},~c^{*} = \argmin_{\boldsymbol{a} \in \mathbb{R}^p,~c\in\mathbb{R}} ~\max \{~ \mathbb{P}( \boldsymbol{a}^{\top} \mathbf{x} > c ~|~ y=1),~\mathbb{P}( \boldsymbol{a}^{\top} \mathbf{x} \leq c ~|~ y=2) ~\}.
\end{equation}
It is not hard to see that the minimax condition requires the cluster-specific error rates $\mathbb{P}( \boldsymbol{a^{*}}^{\top} \mathbf{x} > c^{*} ~|~ y=1)$ and $\mathbb{P}( \boldsymbol{a^{*}}^{\top} \mathbf{x} \leq c^{*} ~|~ y=2)$ to be equal. Consequently, the cluster overlap $\alpha$ becomes
\begin{equation}
\label{eq:overlap}
\alpha = 2 ~ \mathbb{P}( \boldsymbol{a^{*}}^{\top} \mathbf{x} > c^{*} ~|~ y=1).
\end{equation}
Geometrically, our definition means that two clusters overlap at level $\alpha$ if their marginal distributions along the minimax classification axis $\boldsymbol{a}^{*}$ intersect at the $1-\alpha/2$ and $\alpha/2$ quantiles. The left panel of Figure \ref{fig:overlap} highlights the probability mass bounded by these quantiles in gray. 

Note that our formulation of \textit{minimax} classification error explicitly does not take into account class probabilities for the clusters. Equation (\ref{eq:defminimax}) depends only on the class-conditional probabilities. 
Our reasoning is that the underlying reality of each cluster depends on its class-conditional probability distribution, not the class probability.  

\cite{oclus} quantify cluster overlap by computing the full distributional overlap in $p$ dimensions. We prefer our one-dimensional notion in terms of minimax classification error, since high-dimensional Euclidean geometry exhibits a number of counterintuitive phenomena. Specifically, as $p \rightarrow \infty$, the majority of a sphere's volume becomes concentrated in a thin shell from its surface, and so $p$-dimensional overlap goes to zero unless the marginal overlaps in each dimension approach 100\% (\citealp{BlumHD}; \citealp{oclus}). 

\begin{figure}[h!]
\centering
    \includegraphics[width=0.8\textwidth]{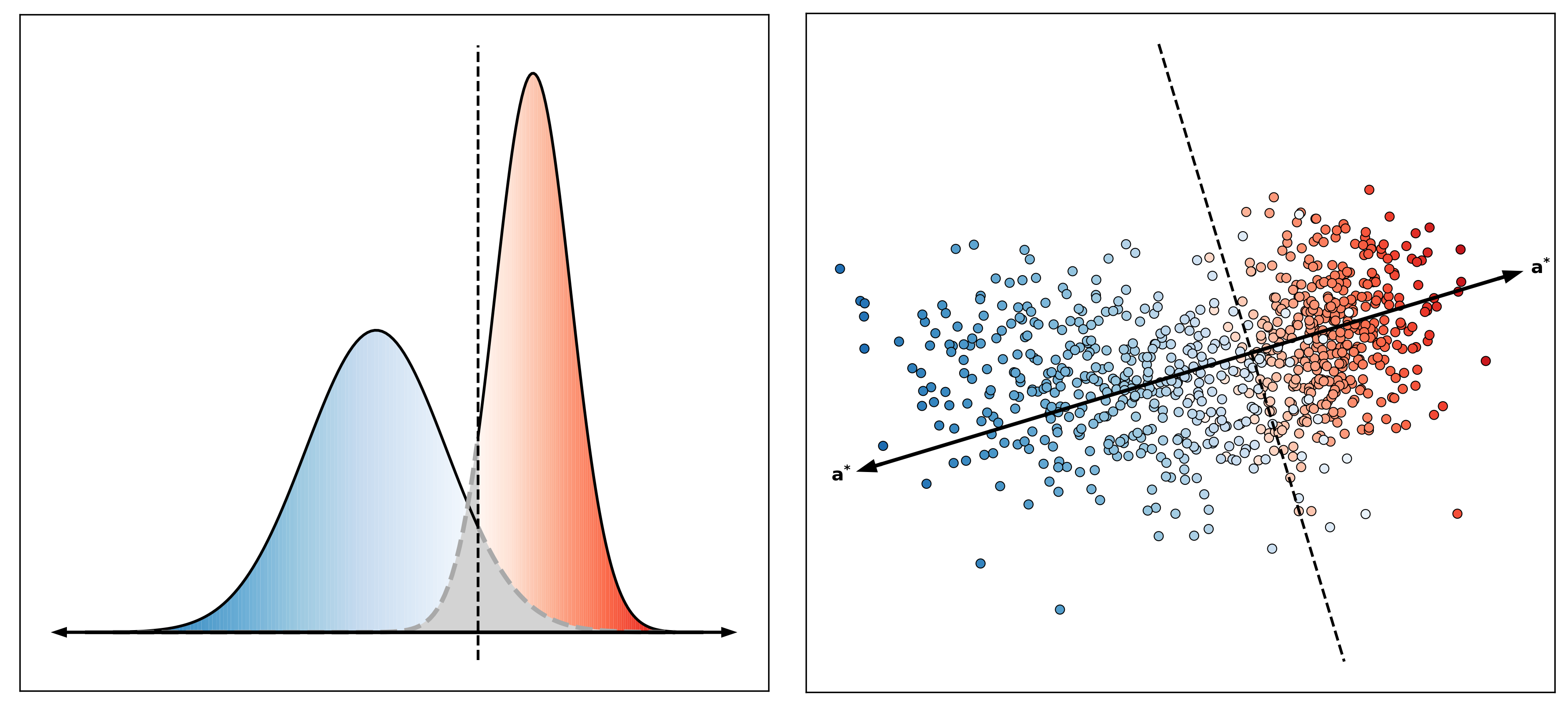}
    \caption{Cluster overlap based on the misclassification rate of the best linear classifier, in 1D (left) and 2D (right). The black dashed lines show the decision boundaries corresponding to minimax classification rules between the blue and red clusters, and the gray shaded areas represent classification errors. Cluster overlap $\alpha$ is the total probability mass of the gray areas. Here, $\alpha = 14.7\%$ for both the left and right panels.}
    \label{fig:overlap}
\end{figure}

Computing our cluster overlap $\alpha$ requires finding the minimax classification axis $\boldsymbol{a}^{*}$. \cite{bahadur} describe an algorithm for computing $\boldsymbol{a}^{*}$ exactly, in the case of multivariate normal distributions. Unfortunately, their method requires computing the matrix inverse $(t\boldsymbol{\Sigma}_1 + (1-t)\boldsymbol{\Sigma}_2)^{-1}$ for $\text{O}(\log(1/\epsilon)$ values of $t$, where $\boldsymbol{\Sigma}_1, \boldsymbol{\Sigma}_2$ are the clusters' covariance matrices and $\epsilon$ is a numerical tolerance. To avoid these matrix inversions, we propose an approximation of the minimax classification axis based on linear discriminant analysis (LDA).

For a pair of multivariate normal clusters with means $\boldsymbol{\mu}_1 \neq \boldsymbol{\mu}_2$ and the same covariance matrix $\boldsymbol{\Sigma}$, the axis  $\boldsymbol{a}_\text{\tiny{LDA}} = \boldsymbol{\Sigma}^{-1}(\boldsymbol{\mu}_2 - \boldsymbol{\mu}_1)$ minimizes the overall misclassification rate, assuming equal class probabilities (see \citealp{ESL}). Unfortunately, the result does not hold for unequal covariance matrices $\boldsymbol{\Sigma}_1 \neq \boldsymbol{\Sigma}_2$. Thus, we propose the approximation $\boldsymbol{a}_\text{\tiny{LDA}} = (\frac{\boldsymbol{\Sigma}_1 + \boldsymbol{\Sigma}_2}{2})^{-1}(\boldsymbol{\mu}_2 - \boldsymbol{\mu}_1)$. Figure \ref{fig:approximate_overlap} verifies that this LDA-based approximation closely matches the exact overlap, as compared with a simpler ``center-to-center'' (C2C) approximation that uses the difference between the cluster centers as the classification axis (i.e., $\mathbf{a}_{\tiny{C2C}} = \boldsymbol{\mu}_2 - \boldsymbol{\mu}_1$). In the figure, each data point corresponds to a pair of multivariate normal clusters with the pairwise overlap shown. The full simulation is based on 900 pairs generated from a variety of data set archetypes with different cluster shapes and numbers of dimensions (see code repository for implementation details).

\begin{figure}[h!]
\centering
    \includegraphics[width=0.47\textwidth]{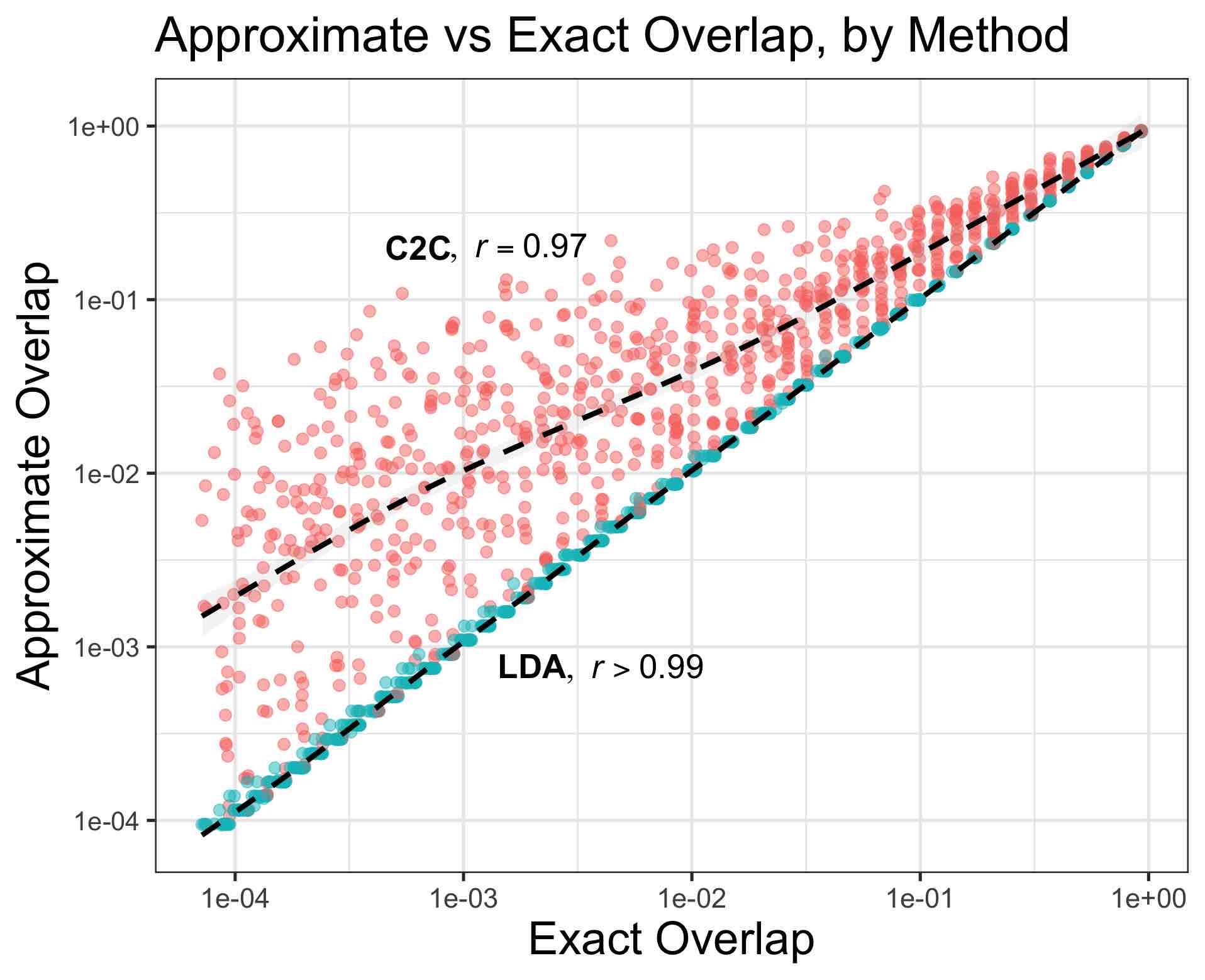}
    \hfill
    \includegraphics[width=0.47\textwidth]{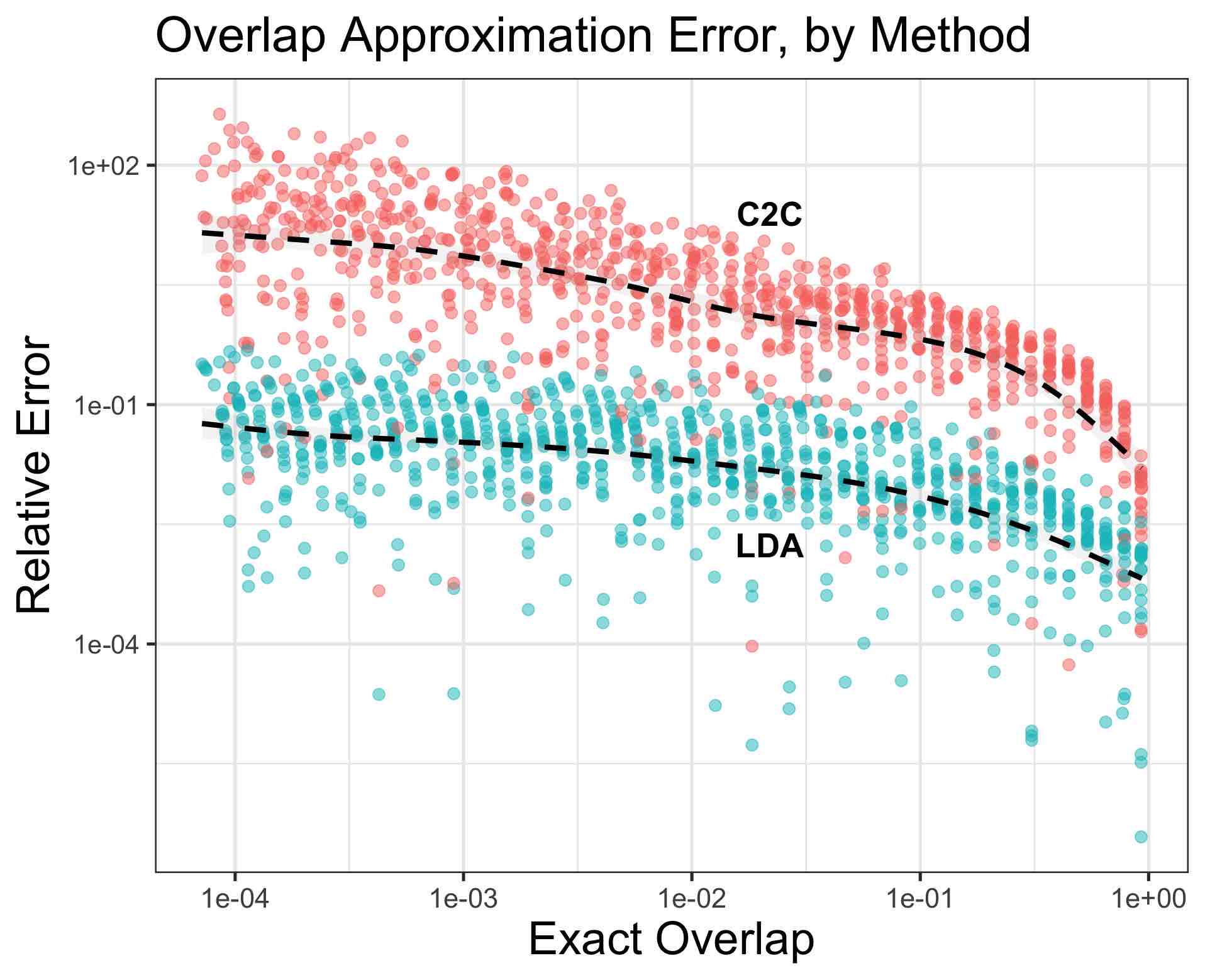}
    \caption{Quality of approximating cluster overlap using our LDA and the simpler center-to-center (C2C) approximations. Both approaches show strong correlations with exact cluster overlap, achieving Pearson correlations $r$ close to 1 (left). However, the C2C method incurs significant relative error, while the LDA approximation typically comes within 10\% of the exact overlap (right). The dashed lines indicate estimated conditional means.}
    \label{fig:approximate_overlap}
\end{figure}

Theorem \ref{thm:ldaoverlap} gives a formula for computing the LDA-approximate cluster overlap. The proof is in Appendix C.

\begin{theorem}[LDA-Based Cluster Overlap]
\label{thm:ldaoverlap}
For two multivariate normal clusters with means $\boldsymbol{\mu}_1 \neq \boldsymbol{\mu}_2$ and covariance matrices $\boldsymbol{\Sigma}_1, \boldsymbol{\Sigma}_2$, the approximate cluster overlap $\alpha_\text{\tiny{LDA}}$ based on the linear separator $\boldsymbol{a}_\text{\tiny{LDA}} = (\frac{\boldsymbol{\Sigma}_1 + \boldsymbol{\Sigma}_2}{2})^{-1}(\boldsymbol{\mu}_2 - \boldsymbol{\mu}_1)$ is 
\begin{equation}
\label{eq:lda_overlap}
\alpha_\text{\tiny{LDA}} = 2\big(1 - \Phi \Big( \frac{\boldsymbol{a}_\text{\tiny{LDA}}^{\top} (\boldsymbol{\mu}_2 - \boldsymbol{\mu}_1)}{\sqrt{\boldsymbol{a}_\text{\tiny{LDA}}^{\top} \boldsymbol{\Sigma}_1 \boldsymbol{a}_\text{\tiny{LDA}}} + \sqrt{\boldsymbol{a}_\text{\tiny{LDA}}^{\top}\boldsymbol{\Sigma}_2\boldsymbol{a}_\text{\tiny{LDA}}}} \Big) \big),
\end{equation}
where $\Phi(z)$ is the cumulative distribution function of the standard normal distribution. 
Moreover, if $\boldsymbol{\Sigma}_1 = \lambda \boldsymbol{\Sigma}_2$ for some $\lambda$ then $\alpha_\text{\tiny{LDA}}$ equals the exact cluster overlap $\alpha$.
\end{theorem}

Theorem \ref{thm:ldaoverlap} shows that cluster overlap depends on quantiles $q$ of the form
\begin{align}
\label{eq:overlap_metric}
    q(\boldsymbol{\mu}_1,\boldsymbol{\mu}_2; \boldsymbol{\Sigma}_1, \boldsymbol{\Sigma}_2; \boldsymbol{a}) & = 
    \frac{\boldsymbol{a}^{\top} (\boldsymbol{\mu}_2 - \boldsymbol{\mu}_1)}{\sqrt{\boldsymbol{a}^{\top} \boldsymbol{\Sigma}_1 \boldsymbol{a}} + \sqrt{\boldsymbol{a}^{\top}\boldsymbol{\Sigma}_2\boldsymbol{a}}},
\end{align}
where $\boldsymbol{a}$ is a classification axis. Since these quantiles are inversely related to cluster overlap, they quantify cluster \textit{separation}.

When generating synthetic data, \texttt{repliclust} uses Theorem \ref{thm:ldaoverlap} even when cluster distributions are non-normal. Figure \ref{fig:viz_nonnormal_overlap} in Appendix D suggests that controlling overlap between non-normal distributions works well, except for distributions with bounded support (such as beta).

\subsection{Adjusting Cluster Centers to Meet Overlap Constraints}

To enforce maximum and minimum overlaps for a mixture model, we optimize a loss function using stochastic gradient descent (SGD). Given mixture model parameters $\boldsymbol{\theta}$ (including the cluster centers, principal axes, and principal axis lengths), the \textit{overlap loss} is
\begin{equation}
\label{eq:overlaploss}
\mathcal{L}(\boldsymbol{\theta}) = \frac{1}{k} \sum_{i=1}^k \ell_i,
\end{equation}
where the loss on the $i$-th cluster is
\begin{equation}
\label{eq:singleclusterloss}
\ell_i = p_{\lambda}\big((\min_{j \neq i} q_{ij} - q_{\text{max}})^{+}\big) + \sum_{j \neq i} p_{\lambda}\big((q_{\text{min}} - q_{ij})^{+}\big).
\end{equation}
Here, $q_\text{min}$, $q_\text{max}$, and $q_{ij}$ measure cluster separation as expressed in Equation (\ref{eq:overlap_metric}): $q_\text{min}$ is the minimum allowed separation (corresponding to overlap $\alpha=\texttt{max\_overlap}$); $q_\text{max}$ is the maximum allowed separation (corresponding to $\alpha=\texttt{min\_overlap}$); and $q_{ij}$ is the separation between the $i$-th and $j$-th clusters. The penalty function $p_{\lambda}$ is the polynomial $p_{\lambda}(x) := \lambda x + (1-\lambda) x^2$, where $\lambda$ is a tuning parameter. Finally, $x^{+} := \max(x,0)$ is a filter that passes on only positive inputs (corresponding to a violation of user-specified constraints).

\subsubsection{Intuition behind the Overlap Loss}
\label{subsec:intuition_overlap_loss}

By design, the loss (\ref{eq:singleclusterloss}) vanishes when the cluster centers, principal axes, and principal axis lengths satisfy the user-specified overlap constraints. The first term penalizes violation of the minimum overlap condition. Indeed, if cluster $i$ is too far away from the other clusters, the separation $\min_{j \neq i} q_{ij}$ between cluster $i$ and its closest neighbor exceeds the maximum allowed separation $q_\text{max}$. A penalty of the excess $(\min_{j \neq i} q_{ij} - q_\text{max})^{+}$ yields the first term in (\ref{eq:singleclusterloss}). The second term measures violation of the maximum overlap condition. If the separation $q_{ij}$ between clusters $i$ and $j$ falls short of the smallest allowed separation $q_\text{min}$, the shortfall $(q_\text{min} - q_{ij})^{+}$ incurs a penalty that serves to push these clusters apart.

The penalty $p_{\lambda}$ in (\ref{eq:singleclusterloss}) ranges from quadratic to linear based on the value of $\lambda$. Keeping the penalty partly linear ($\lambda > 0$) helps gradient descent drive the overlap loss to \textit{exactly} zero because a purely quadratic loss would result in a vanishing derivative when overlap constraints approach satisfaction.

\subsubsection{Running the Minimization in Practice}
\label{subsec:sgd_init}

When minimizing (\ref{eq:overlaploss}) using SGD, we initially place cluster centers randomly within a sphere. The volume $V$ of this sphere influences the initial overlaps between clusters. To select an appropriate value, we fix the ratio $\rho$ of the sum of cluster volumes to $V$; essentially, $\rho$ is the density of clusters within the sampling volume. Values of $\rho$ around 10\% work well in low dimensions. In higher dimensions, however, results from the mathematics of sphere-packing motivate a downward adjustment. A lower bound by \cite{Ball} states that the maximum achievable density when placing non-overlapping spheres inside $\mathbb{R}^p$ is at least $p2^{1-p}$. Thus, in $p$ dimensions we use an adjusted density $\rho^\text{adj}$ defined by
\begin{equation*}
\rho^\text{adj}(p) = p2^{1-p} \rho^{2D},
\end{equation*}
where $\rho^{2D} \approx 10\%$ is the equivalent density in 2D.

Following initialization, we optimize the cluster centers $\{\boldsymbol{\mu}_i\}_{i=1}^{k}$ using stochastic gradient descent. During this process, the principal axes and their lengths are fixed. Each iteration performs the update
\begin{equation}
\label{eq:sgd}
\big[ \boldsymbol{\mu}_1 | \boldsymbol{\mu}_2 |~ ... ~| \boldsymbol{\mu}_k \big] ~ \leftarrow ~ \big[\boldsymbol{\mu}_1 | \boldsymbol{\mu}_2 |~ ... ~| \boldsymbol{\mu}_k\big] ~-~ \eta~ \big[\nabla_{\boldsymbol{\mu}_1} \ell_i | \nabla_{\boldsymbol{\mu}_2} \ell_i |~ ... ~| \nabla_{\boldsymbol{\mu}_k} \ell_i \big]
\end{equation}
on the single-cluster loss $\ell_i$, where $\eta$ is the learning rate. For each epoch, we randomly permute the order $i = 1,2,...,k$ of clusters and apply the updates (\ref{eq:sgd}) in turn for each cluster.

Experiments suggest that our minimization procedure drives the overlap loss to zero at an exponential rate (linear convergence rate), as expected for gradient descent. The number of epochs required seems largely independent of the number of clusters, though it increases slightly with the number of dimensions.

\subsection{Making Cluster Shapes More Irregular}
\label{sec:realism}

In many application domains, ellipsoidal center-based clusters are good models for the data. This is often the case when clusters can be characterized by a ``typical'' element. For example, in single-cell RNA sequencing, \cite{PopAlign} represent the differences in transcriptional activity across cell types using Gaussian mixture models. In this case, the cluster centers correspond to each cell type's typical gene expression profile. Similarly, in deep learning, a large language model's hidden states often form compact clusters that can effectively classify text (\citealp{UniversalLMClassification}), assess the truthfulness of a response (\citealp{AzariaMitchell2023}), and help detect out-of-distribution inputs (\citealp{RenOODforLLMs}).

However, convex clusters can be inappropriate. For example, \cite{DBSCAN} developed the density-based DBSCAN algorithm specifically to handle spatial (GPS) data. Characterized by thin loops and irregular shapes, such data does not fit a center-based clustering paradigm at all. Another example is directional data lying on a sphere (\citealp{ClusteringVMF}; \citealp{DirectionalCoClustering}). In this case, the clusters have centers but are not convex.

To accommodate use cases where ellipsoidal clusters are inappropriate, \texttt{repliclust} provides two post-processing functions for creating non-convex clusters. The first, \\ \texttt{repliclust.distort}, passes a data set through a randomly initialized neural network, as shown in Figure \ref{fig:distort}. The second, \texttt{repliclust.wrap\_around\_sphere}, wraps a $p$-dimensional data set around the $(p+1)$-dimensional sphere by applying an inverse stereographic projection, as shown in Figure \ref{fig:wrap_around_sphere}. Appendix D describes the default architecture of the neural network used in \texttt{repliclust.distort}.

\begin{figure}[h!]
\centering
    \begin{subfigure}{0.47\textwidth}
        \centering
        \includegraphics[width=\textwidth]{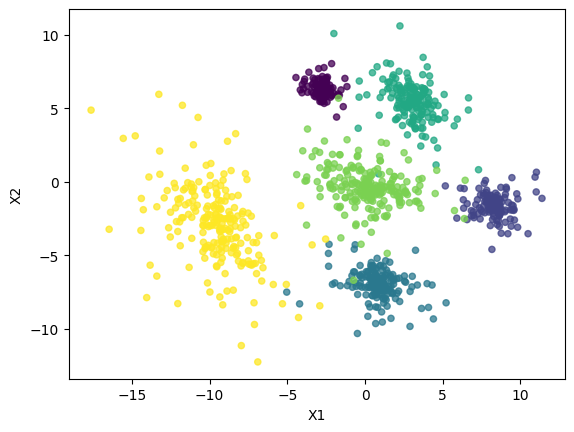}
        \caption{Before \texttt{distort}}
    \end{subfigure}
    \hfill
    \begin{subfigure}{0.47\textwidth}
        \centering
        \includegraphics[width=\textwidth]{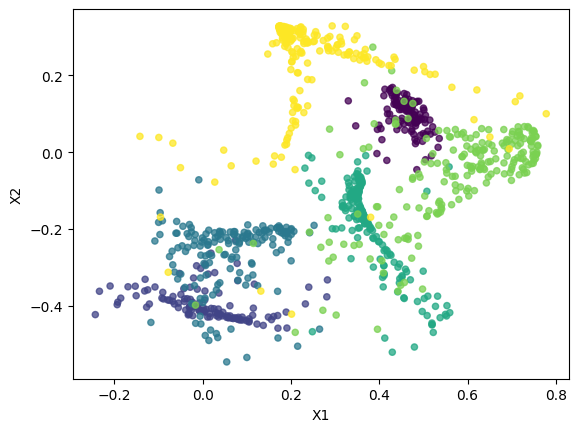}
        \caption{After \texttt{distort}}
    \end{subfigure}
    \caption{You can create non-convex, irregularly shaped clusters by applying the \texttt{distort} function, which runs your dataset through a randomly initialized neural network.}
    \label{fig:distort}
\end{figure}

\begin{figure}[h!]
\centering
    \begin{subfigure}{0.47\textwidth}
        \centering
        \includegraphics[width=\textwidth]{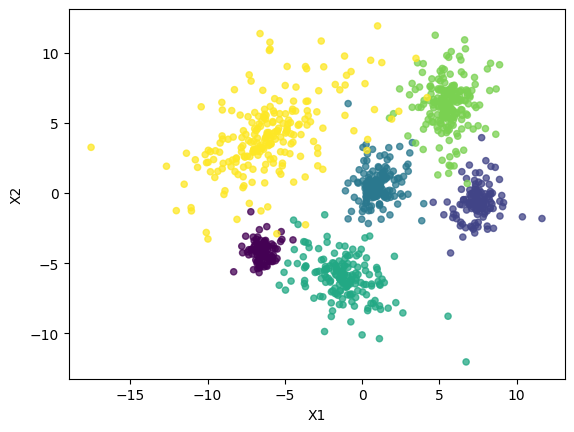}
        \caption{Before \texttt{wrap\_around\_sphere}}
    \end{subfigure}
    \hfill
    \begin{subfigure}{0.47\textwidth}
        \centering
        \includegraphics[width=\textwidth]{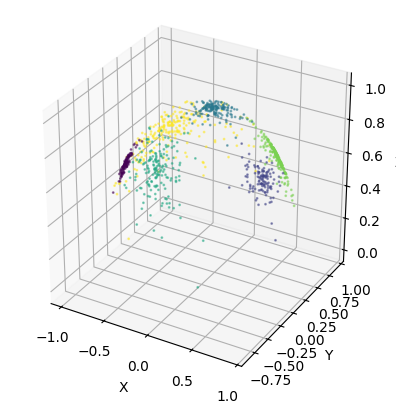}
        \caption{After \texttt{wrap\_around\_sphere}}
    \end{subfigure}
    \caption{You can create directional data by wrapping your dataset around the sphere using the \texttt{wrap\_around\_sphere} function. The function works in arbitrary dimensions.}
    \label{fig:wrap_around_sphere}
\end{figure}

\subsection{Generating Archetypes from Natural Language}

To fully realize the potential of a high-level, archetype-based approach to synthetic data generation, our package \texttt{repliclust} draws on the OpenAI API (\citealp{openai_api}) to allow users to create data set archetypes directly from verbal descriptions in English.

To map a user's natural language input to the high-level geometric parameters of Table \ref{tbl:archetype_params}, we use few-shot prompting of OpenAI's GPT-4o large language model model (\citealp{GPT3Paper}; \citealp{GPT4Report}). We use this same approach to automatically create short identifiers for archetypes, such as ``ten\_highly\_oblong\_very\_different\_shapes\_moderate\_overlap''. These identifiers contain no white space and obey the naming rules for Python variables.

Version 1.0.0. of \texttt{repliclust} uses 22-shot prompting to map archetype descriptions to parameters, and 11-shot prompting to generate archetype identifiers, which appears to work well. In the rare case that the natural language workflow fails, \texttt{repliclust} throws an error. We release all few-shot examples and prompt templates in the \texttt{natural\_language} module of our code base. They are also reproduced in Appendix F.

\section{Results}
\label{sec:results}

In this section, we first illustrate the convenience of archetype-based data generation by conducting a mock benchmark between several clustering algorithms. Second, we verify that our notion of cluster overlap effectively captures clustering difficulty (even in higher dimensions), and study the distribution of pairwise overlaps in data sets with multiple clusters.

\subsection{Mock Benchmark}
\label{sec:benchmark}

We demonstrate the convenience and informativeness of our archetype-based data generator by running a mock benchmark comparing several clustering algorithms on data drawn from different archetypes. The benchmark is not meant to comprehensively evaluate the strengths and weaknesses of the clustering algorithms we consider. Instead, our goal is to demonstrate the advantages (in convenience and informativeness) of an archetype-based approach to synthetic data generation.

First, we construct six archetypes by providing \texttt{repliclust} with the following verbal descriptions: 
\begin{enumerate}
    \item \texttt{`twelve clusters of different distributions'}
    \item \texttt{`twelve clusters of different distributions and high class imbalance'}
    \item \texttt{`seven highly separated clusters in 10D with very different shapes'}
    \item \texttt{`seven clusters in 10D with very different shapes and significant overlap'}
    \item \texttt{`four clusters in 100D with 100 samples each'}
    \item \texttt{`four clusters in 100D with 1000 samples each'}
\end{enumerate}
\texttt{Repliclust} maps these descriptions to the following data set archetypes:
\begin{enumerate}
\item
\begin{lstlisting}
{`name': `twelve_clusters_different_distributions', `n_clusters': 12, `dim': 2, `n_samples': 1200, `aspect_ref': 1.5, `aspect_maxmin': 2, `radius_maxmin': 3, `imbalance_ratio': 2, `max_overlap': 0.05, `min_overlap': 0.001, `distributions': [`normal', `exponential', `gamma', `weibull', `lognormal']}
\end{lstlisting}

\item
\begin{lstlisting}
{`name': `twelve_different_distributions_high_class_imbalance', `n_clusters': 12, `dim': 2, `n_samples': 1200, `aspect_ref': 1.5, `aspect_maxmin': 2, `radius_maxmin': 3, `imbalance_ratio': 5, `max_overlap': 0.05, `min_overlap': 0.001, `distributions': [`normal', `exponential', `gamma', `weibull', `lognormal']}
\end{lstlisting}

\item
\begin{lstlisting}
{`name': `seven_highly_separated_10d_very_different_shapes', `n_clusters': 7, `dim': 10, `n_samples': 700, `aspect_ref': 1.5, `aspect_maxmin': 3.0, `radius_maxmin': 3.0, `imbalance_ratio': 2, `max_overlap': 0.0001, `min_overlap': 1e-05, `distributions': [`normal', `exponential']}
\end{lstlisting}

\item
\begin{lstlisting}
{`name': `seven_very_different_shapes_significant_overlap_10d', `n_clusters': 7, `dim': 10, `n_samples': 700, `aspect_ref': 1.5, `aspect_maxmin': 3.0, `radius_maxmin': 3, `imbalance_ratio': 2, `max_overlap': 0.2, `min_overlap': 0.1, `distributions': [`normal', `exponential']}
\end{lstlisting}

\item
\begin{lstlisting}
{`name': `four_clusters_100d_100_samples_each', `n_clusters': 4, `dim': 100, `n_samples': 400, `aspect_ref': 1.5, `aspect_maxmin': 2, `radius_maxmin': 3, `imbalance_ratio': 2, `max_overlap': 0.05, `min_overlap': 0.001, `distributions': [`normal', `exponential']}
\end{lstlisting}

\item
\begin{lstlisting}
{`name': `four_clusters_100d_1000_samples_each', `n_clusters': 4, `dim': 100, `n_samples': 4000, `aspect_ref': 1.0, `aspect_maxmin': 1.0, `radius_maxmin': 3, `imbalance_ratio': 2, `max_overlap': 0.05, `min_overlap': 0.001, `distributions': [`normal', `exponential']}
\end{lstlisting}
\end{enumerate}
Note that the automatically generated names for archetype 5. and 6. did not end in the suffix ``\_each.'' We added this suffix for clarity.

\subsubsection{Examining the Generated Data Sets}

Figure \ref{fig:archetype_untf} shows four representative data sets with convex clusters drawn from each archetype. Figure \ref{fig:archetype_tf} shows non-convex clusters resulting from applying the \texttt{distort} function (as described in Section \ref{sec:realism}). For archetypes with dimensionality greater than two, we use t-SNE to visualize the data sets in 2D (\citealp{tSNEPaper}). 

The figures show that \texttt{repliclust} effectively generates data sets with similar geometric characteristics. Moreover, the data sets match up well with the user-specified verbal descriptions. Applying \texttt{distort} seems to make some clusters very long and thin, but otherwise results in satisfying non-convex clusters. The \texttt{distort} function seems to roughly preserve cluster overlaps; we leave a careful study of this to future work. The 10 and 100-dimensional archetypes indicate that our overlap control effectively handles different numbers of dimensions.

\begin{figure}[htbp]
    \centering
    \includegraphics[width=0.85\textwidth]{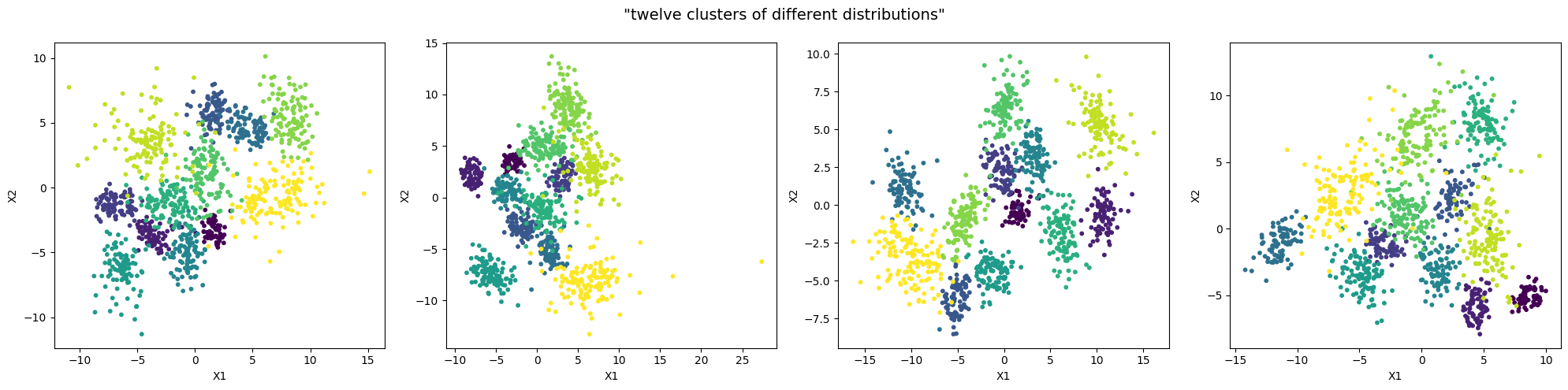}
    \includegraphics[width=0.85\textwidth]{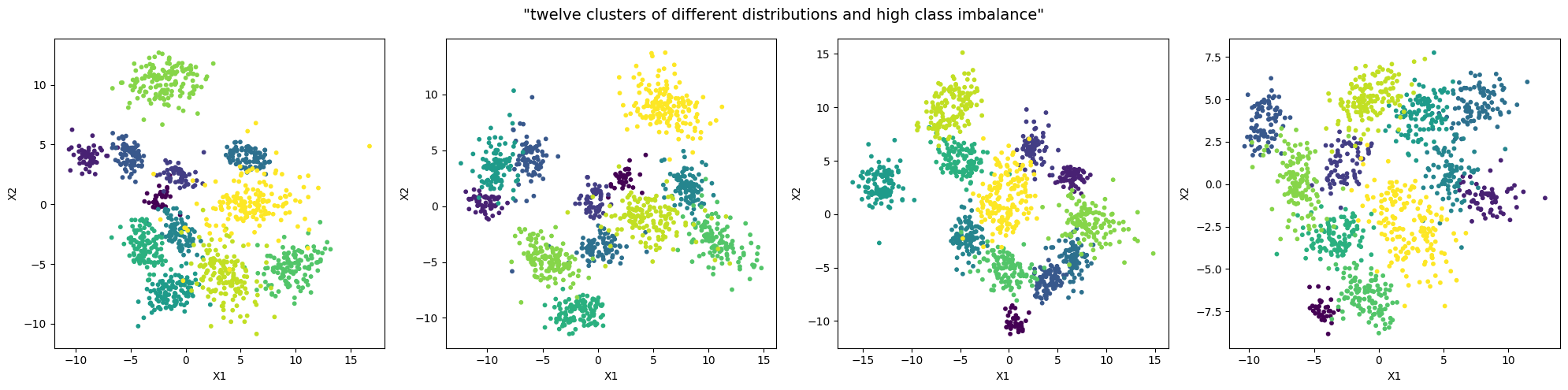}
    \includegraphics[width=0.85\textwidth]{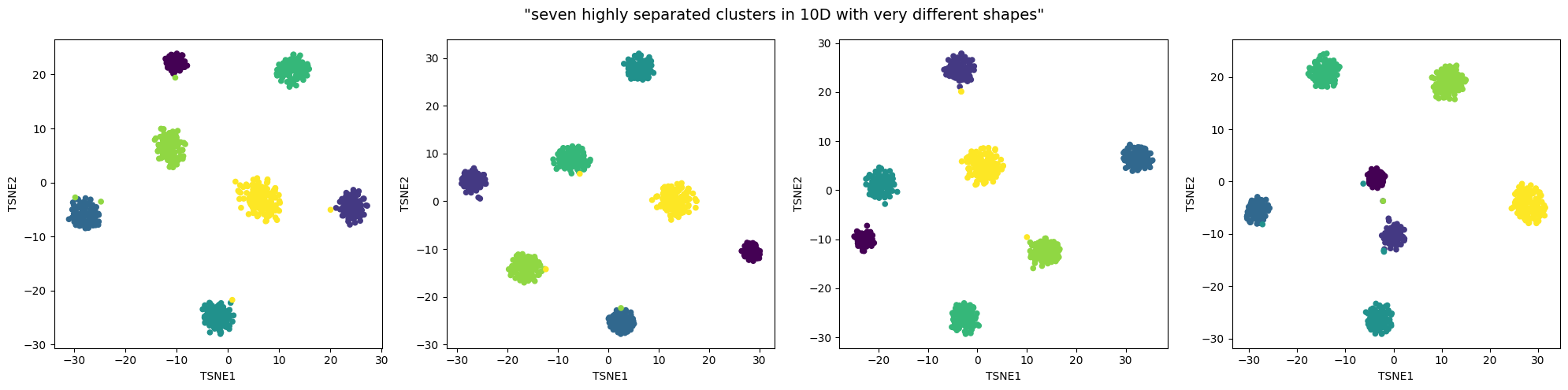}
    \includegraphics[width=0.85\textwidth]{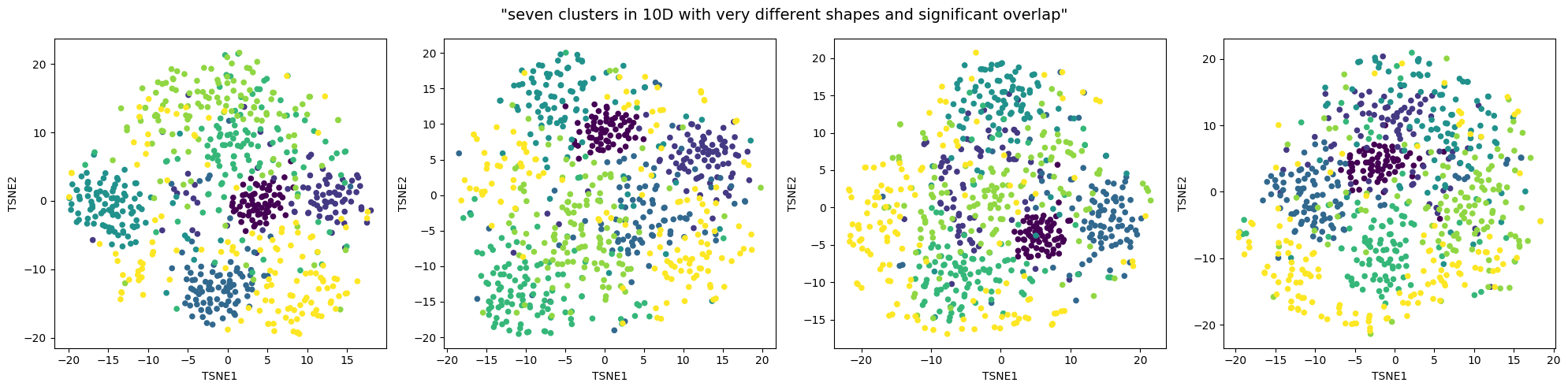}
    \includegraphics[width=0.85\textwidth]{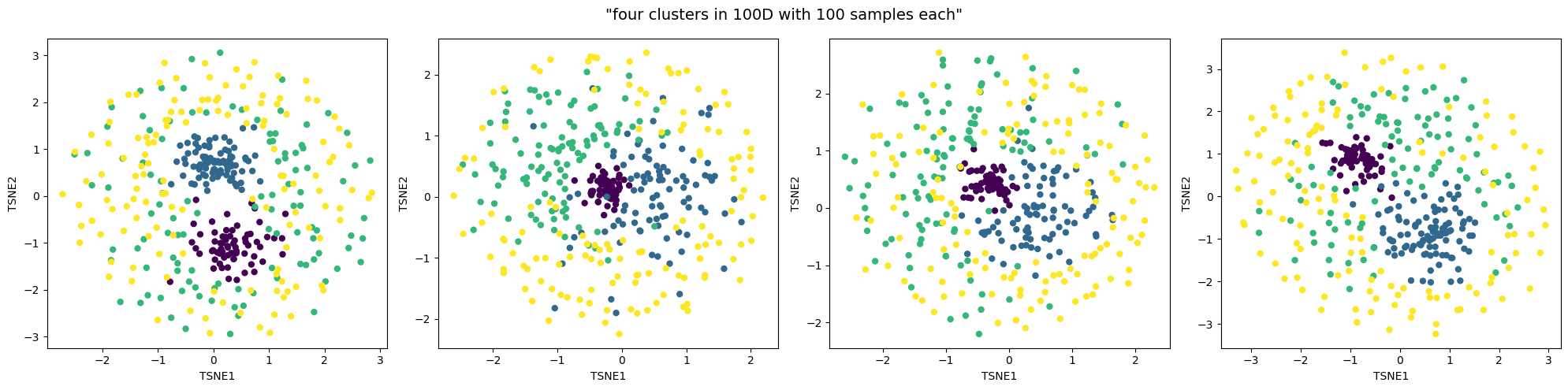}
    \includegraphics[width=0.85\textwidth]{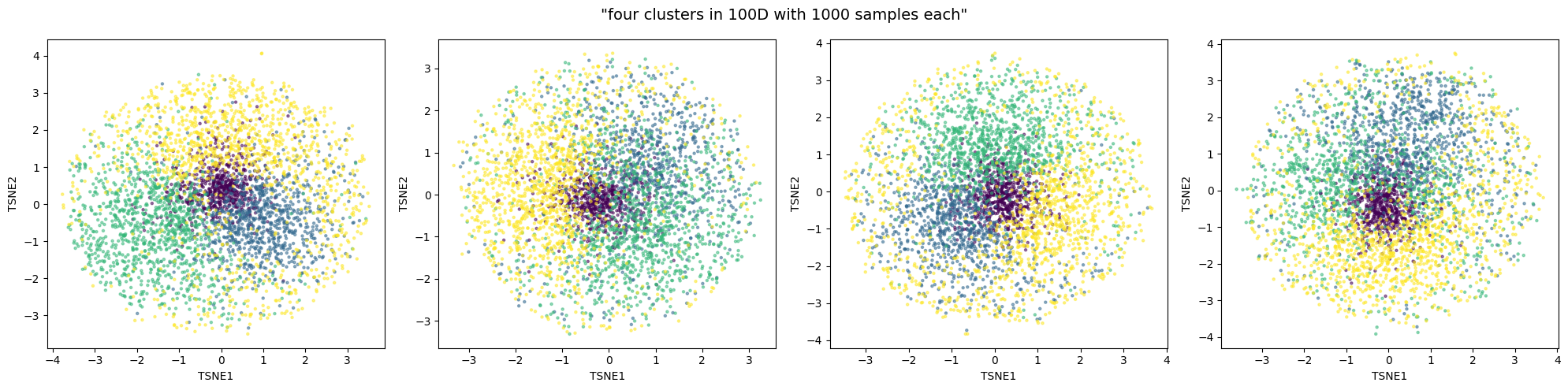}
    \caption{Representative convex clusters drawn from the archetypes in our benchmark (not cherry-picked).}
    \label{fig:archetype_untf}
\end{figure}

\begin{figure}[htbp]
    \centering
    \includegraphics[width=0.85\textwidth]{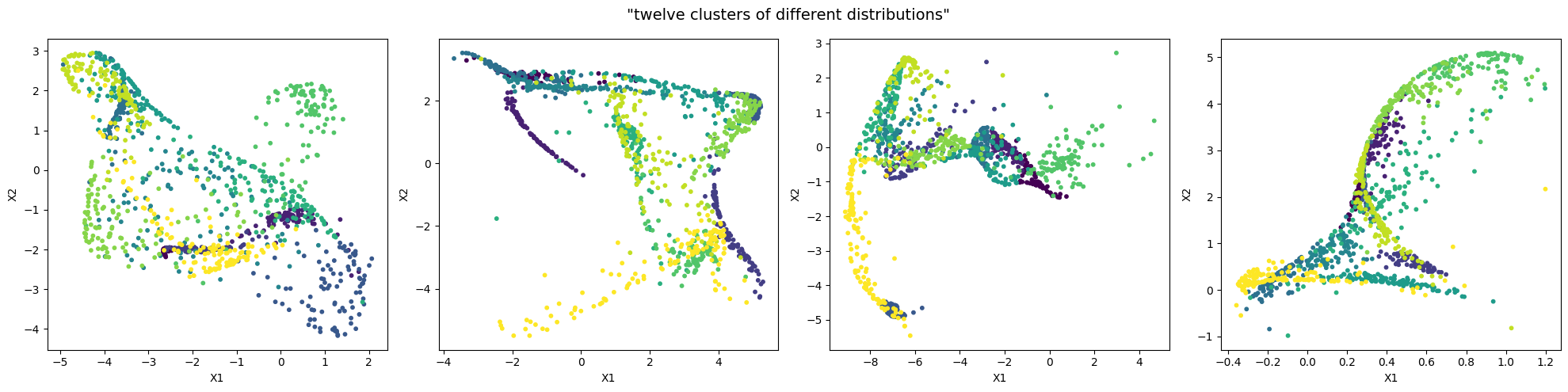}
    \includegraphics[width=0.85\textwidth]{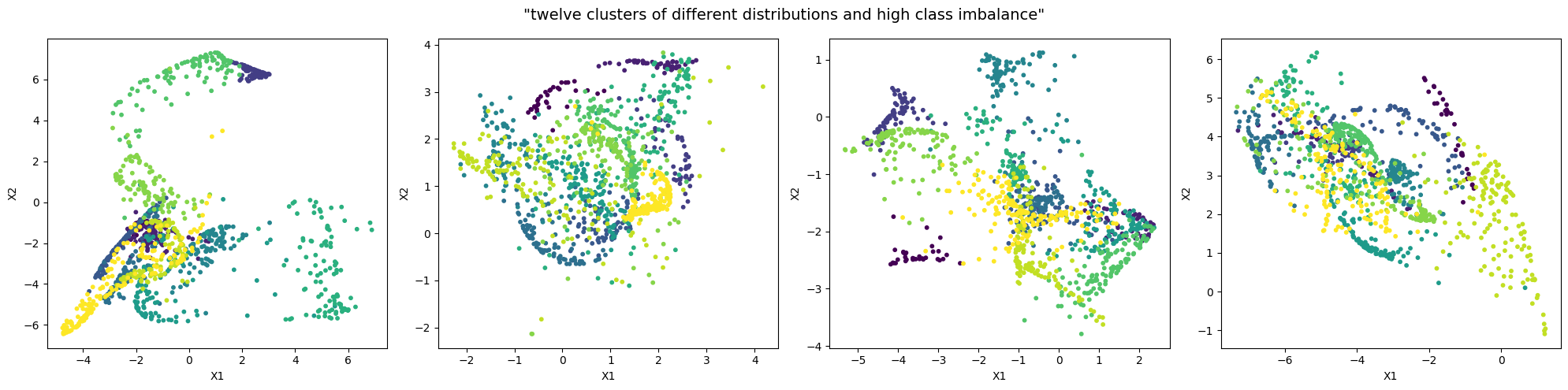}
    \includegraphics[width=0.85\textwidth]{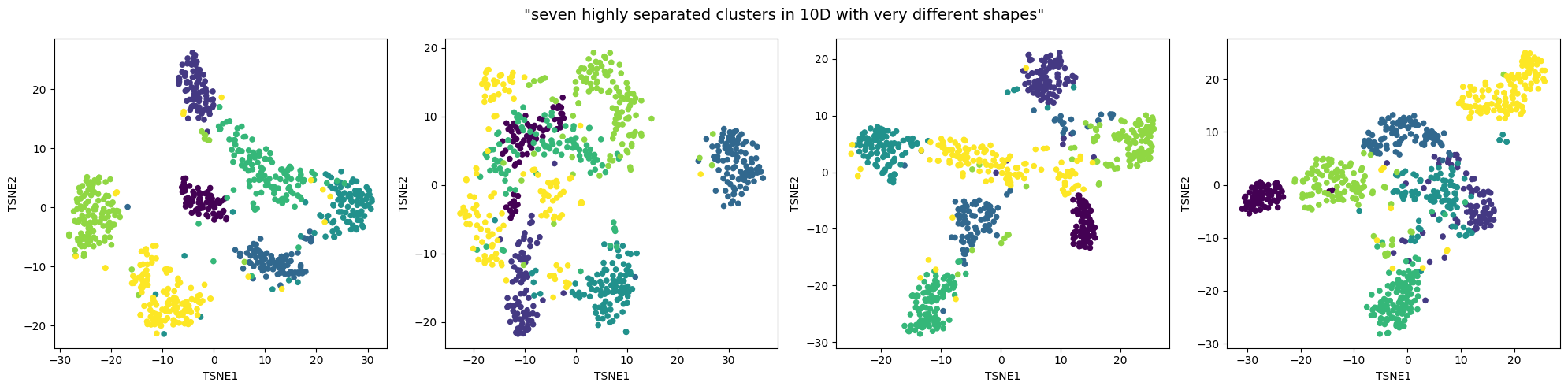}
    \includegraphics[width=0.85\textwidth]{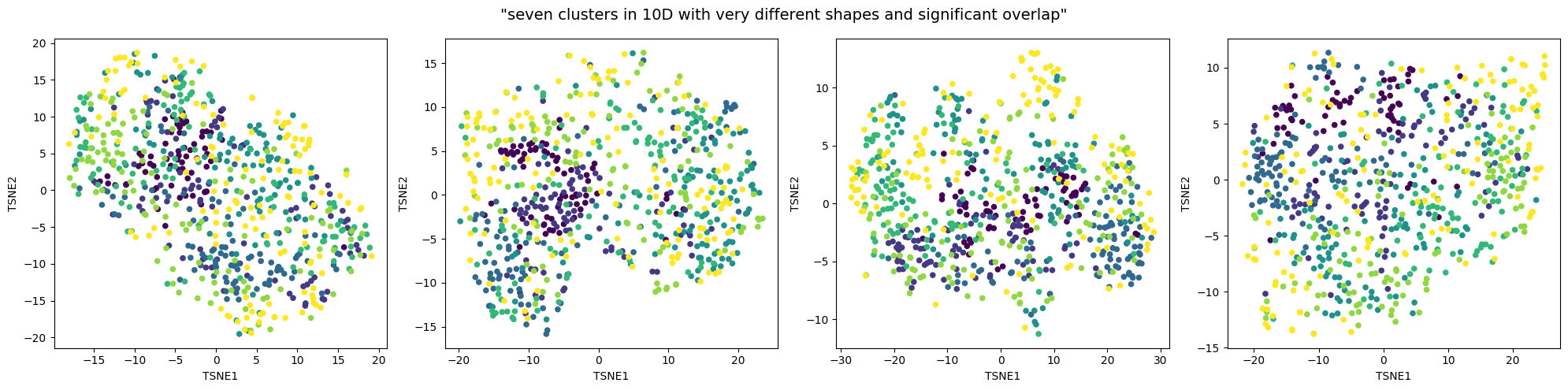}
    \includegraphics[width=0.85\textwidth]{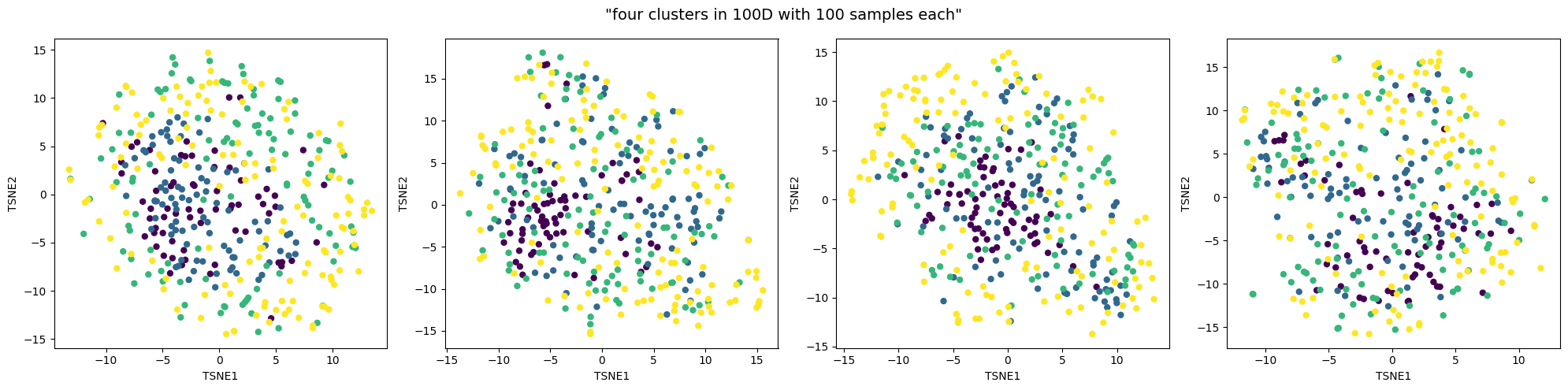}
    \includegraphics[width=0.85\textwidth]{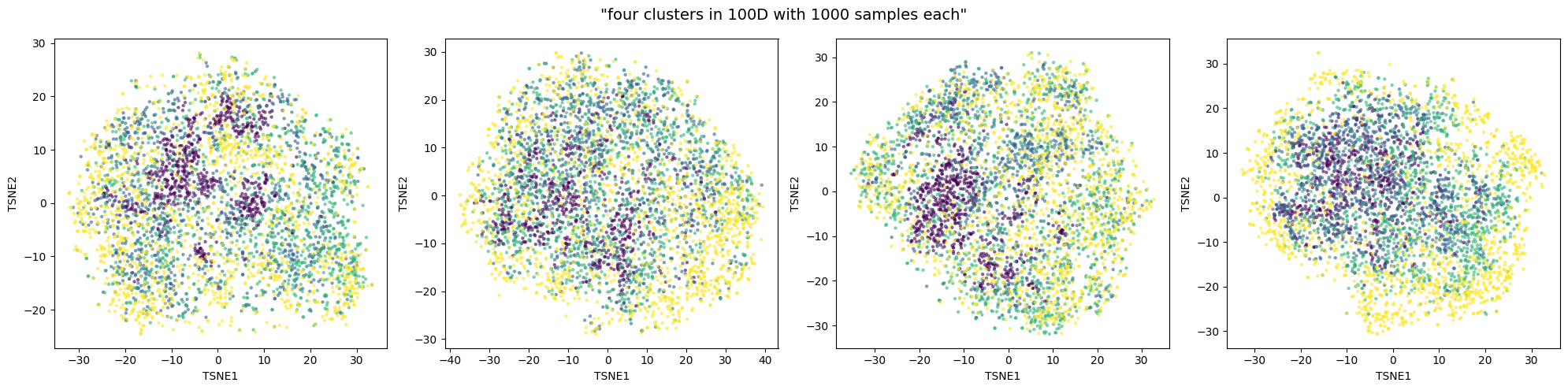}
    \caption{Representative non-convex clusters drawn from the archetypes in our benchmark (not cherry-picked).}
    \label{fig:archetype_tf}
\end{figure}

\subsubsection{Comparing the Clustering Algorithms}

\nocite{em}
\nocite{kmeans}
\nocite{SpectralNg}
\nocite{SpectralShiMalik}
\nocite{hdbscan}

We compare the following clustering algorithms: K-Means, hierarchical (with Ward linkage), spectral (with radial-basis function affinity), HDBSCAN, and expectation maximization for a Gaussian mixture model (EM-GMM). We originally intended to include DBSCAN as well. However, the heuristics we tried for choosing the neighborhood radius (including what is suggested in \citealp{dbscanheuristics}) did not work well across the range of dimensionalities we consider, resulting in too many noise points (see also the discussion by \citealp{dbscanrevisitedrevisited}).

We measure performance in terms of adjusted mutual information (AMI) and adjusted Rand index (ARI), based on the ground truth cluster labels (\citealp{AMIPaper}; \citealp{ARIPaper}). To carry out the benchmark, we sample 10 times from each archetype, resulting in 60 distinct data sets. We repeat this process twice to evaluate performance on convex and non-convex clusters (the latter resulting from applying the \texttt{distort} function on new data sets). Thus, the full benchmark is based on 120 distinct data sets. Table \ref{tbl:benchmark_results} shows the results for K-Means, hierarchical, spectral, and EM-GMM. All these algorithms receive the true number of clusters as an input.

Table \ref{tbl:hdbscan_results} separately lists the performance for HDBSCAN, which we ran using the scikit-learn implementation (\citealp{scikit-learn}) with \texttt{min\_samples}=5. We present the HDBSCAN results separately for two reasons. First, as a density-based algorithm, HDBSCAN cannot make use of the true number of clusters, putting it at a disadvantage. Second, HDBSCAN reports a large number of noise points ($\approx$60\% on average), so that its results are not directly comparable to that of the other algorithms. We report the performance of HDBSCAN based only on the non-noise points. This leads to higher performance numbers, which may compensate for the disadvantage HDBSCAN faces in not knowing the true number of clusters.

\begin{table*}[ht]
\centering
\caption{Benchmark results on convex and non-convex data. The best performance for each archetype is printed in \textbf{bold}.}
\label{tab:clustering_comparison}
\resizebox{\textwidth}{!}{%
\begin{tabular}{|l|cccc|cccc|}
\hline
\textbf{Archetype} & \multicolumn{4}{c|}{\textbf{Convex}} & \multicolumn{4}{c|}{\textbf{Non-Convex}} \\
\hline
                  & \multicolumn{4}{c|}{\textbf{AMI}} & \multicolumn{4}{c|}{\textbf{AMI}} \\
                  & EM-GMM & K-Means & Spectral & Hierarchical & EM-GMM & K-Means & Spectral & Hierarchical \\
\hline
twelve\_clusters\_different\_distributions & \textbf{0.884 (.007)} & 0.848 (.008) & 0.859 (.010) & 0.853 (.007) & \textbf{0.576 (.017)} & 0.563 (.012) & 0.565 (.015) & 0.565 (.011) \\
twelve\_different\_distributions\_high\_class\_imbalance & \textbf{0.854 (.008)} & 0.836 (.007) & 0.852 (.010) & 0.834 (.009) & 0.544 (.021) & 0.542 (.023) & 0.547 (.025) & \textbf{0.553 (.024)} \\
seven\_highly\_separated\_10d\_very\_different\_shapes & 0.983 (.009) & 0.976 (.005) & 0.986 (.003) & \textbf{0.990 (.002)} & \textbf{0.750 (.035)} & 0.606 (.027) & 0.709 (.027) & 0.637 (.029) \\
seven\_very\_different\_shapes\_significant\_overlap & 0.349 (.019) & \textbf{0.452 (.009)} & 0.443 (.014) & 0.380 (.010) & \textbf{0.173 (.014)} & 0.149 (.008) & 0.165 (.010) & 0.145 (.010) \\
four\_clusters\_100d\_100\_samples\_each & 0.063 (.008) & 0.512 (.028) & 0.074 (.023) & \textbf{0.654 (.022)} & 0.077 (.011) & \textbf{0.100 (.012)} & 0.079 (.008) & 0.092 (.008) \\
four\_clusters\_100d\_1000\_samples\_each & 0.548 (.061) & 0.664 (.023) & 0.205 (.014) & \textbf{0.868 (.025)} & \textbf{0.307 (.030)} & 0.145 (.014) & 0.109 (.012) & 0.132 (.016) \\
\hline
\textbf{Average} & 0.614 (.044) & 0.715 (.025) & 0.570 (.046) & \textbf{0.763 (.026)} & \textbf{0.404 (.032)} & 0.351 (.030) & 0.362 (.033) & 0.354 (.031) \\
\hline
                  & \multicolumn{4}{c|}{\textbf{ARI}} & \multicolumn{4}{c|}{\textbf{ARI}} \\
                  & EM-GMM & K-Means & Spectral & Hierarchical & EM-GMM & K-Means & Spectral & Hierarchical \\
\hline
twelve\_clusters\_different\_distributions & \textbf{0.855 (.011)} & 0.795 (.012) & 0.807 (.018) & 0.798 (.014) & 0.368 (.021) & \textbf{0.365 (.015)} & 0.364 (.015) & 0.362 (.012) \\
twelve\_different\_distributions\_high\_class\_imbalance & \textbf{0.800 (.019)} & 0.760 (.012) & 0.787 (.020) & 0.756 (.017) & 0.341 (.022) & 0.347 (.021) & 0.345 (.025) & \textbf{0.360 (.024)} \\
seven\_highly\_separated\_10d\_very\_different\_shapes & 0.967 (.020) & 0.977 (.005) & 0.988 (.003) & \textbf{0.992 (.002)} & \textbf{0.684 (.045)} & 0.505 (.035) & 0.628 (.036) & 0.523 (.039) \\
seven\_very\_different\_shapes\_significant\_overlap & 0.238 (.023) & 0.375 (.012) & \textbf{0.385 (.016)} & 0.291 (.015) & \textbf{0.109 (.011)} & 0.090 (.006) & 0.089 (.007) & 0.081 (.007) \\
four\_clusters\_100d\_100\_samples\_each & 0.004 (.007) & 0.404 (.036) & 0.053 (.014) & \textbf{0.554 (.028)} & 0.018 (.006) & \textbf{0.062 (.011)} & 0.053 (.008) & 0.051 (.005) \\
four\_clusters\_100d\_1000\_samples\_each & 0.408 (.048) & 0.601 (.040) & 0.185 (.011) & \textbf{0.874 (.032)} & \textbf{0.252 (.029)} & 0.088 (.014) & 0.065 (.012) & 0.072 (.016) \\
\hline
\textbf{Average} & 0.545 (.047) & 0.652 (.030) & 0.534 (.044) & \textbf{0.711 (.031)} & \textbf{0.295 (.029)} & 0.243 (.023) & 0.264 (.025) & 0.241 (.025) \\
\hline
\end{tabular}%
}
\label{tbl:benchmark_results}
\end{table*}

\begin{table*}[ht]
\centering
\caption{Benchmark results for HDBSCAN on convex and non-convex data. Numbers that outperform the algorithms in Table \ref{tbl:benchmark_results}, with less than 40\% noise points, are printed in \textbf{bold}.}
\label{tab:clustering_results}
\resizebox{\textwidth}{!}{%
\begin{tabular}{|l|ccc|ccc|}
\hline
\textbf{Archetype} & \multicolumn{3}{c|}{\textbf{Convex}} & \multicolumn{3}{c|}{\textbf{Non-Convex}} \\
\hline
                  & \textbf{AMI} & \textbf{ARI} & \textbf{$\mathbf{p}_\text{noise}$} & \textbf{AMI} & \textbf{ARI} & \textbf{$\mathbf{p}_\text{noise}$} \\
\hline
twelve\_clusters\_different\_distributions & 0.796 (.045) & 0.705 (.066) & 0.269 (.033) & \textbf{0.605 (.016)} & \textbf{0.400 (.026)} & 0.353 (.007) \\
twelve\_different\_distributions\_high\_class\_imbalance & 0.733 (.054) & 0.572 (.083) & 0.238 (.042) & \textbf{0.560 (.024)} & \textbf{0.356 (.036)} & 0.346 (.016) \\
seven\_highly\_separated\_10d\_very\_different\_shapes & \textbf{0.992 (.006)} & 0.984 (.014) & 0.175 (.014) & 0.747 (.106) & 0.700 (.117) & 0.417 (.043) \\
seven\_very\_different\_shapes\_significant\_overlap & 0.595 (.110) & 0.617 (.124) & 0.909 (.020) & 0.163 (.050) & 0.149 (.066) & 0.744 (.072) \\
four\_clusters\_100d\_100\_samples\_each & 1.000 (.000) & 1.000 (.000) & 1.000 (.000) & 0.794 (.137) & 0.804 (.131) & 0.990 (.007) \\
four\_clusters\_100d\_1000\_samples\_each & 0.800 (.133) & 0.800 (.133) & 0.999 (.000) & 0.031 (.018) & 0.023 (.016) & 0.687 (.070) \\
\hline
\textbf{Average} & 0.819 (.035) & 0.780 (.040) & 0.599 (.049) & 0.483 (.047) & 0.405 (.047) & 0.590 (.036) \\
\hline
\end{tabular}%
}
\label{tbl:hdbscan_results}
\end{table*}

The results show that hierarchical clustering performs best on the convex cluster shapes, as long as there is sufficient separation between clusters. On the non-convex clusters, EM-GMM exhibits the strongest performance even though the clusters are no longer multivariate normal after applying \texttt{distort} (see Figure \ref{fig:archetype_tf}). 

K-Means and hierarchical clustering both hold up well on the high-dimensional data, including in the low-sample regime of 100 samples per cluster in 100D. By contrast, EM-GMM dramatically benefits from more samples on the high-dimensional data.

Spectral clustering does not display competitive performance in this benchmark.  While it improves over K-Means in some scenarios, it delivers weaker performance in high dimensions and never performs best across all algorithms.

HDBSCAN shows great difficulty in handling high dimensionality or significant overlap between clusters. However, the algorithm shows strong performance on the 12 non-convex clusters drawn from diverse distributions. We note again that HDBSCAN does not have access to the true number of clusters, in contrast with the other algorithms.

\subsection{Minimax Classification Error Captures Clustering Difficulty}

In Section \ref{sec:overlap}, we defined the overlap between two clusters in terms of the error rate of the best minimax linear classifier. We verify that this notion of overlap conveys clustering difficulty by measuring clustering performance on data sets with different degrees of overlap. For this simulation, we consider data sets with two clusters drawn from an archetype we described as ``two clusters with very different shapes in $p$D''. This verbal description yields an archetype with parameters
\begin{lstlisting}
{ `name': `two_very_different_shapes_|$p$|d', `n_clusters': 2, `dim': |$p$|, `n_samples': 200, `aspect_ref': 1.5, `aspect_maxmin': 3.0, `radius_maxmin': 3, `imbalance_ratio': 2, `max_overlap': 0.05, 'min_overlap': 0.001, `distributions': [`normal', `exponential'] },
\end{lstlisting}
where the dimensionality $p$ ranges across $[2, 10, 30, 100]$. We vary \texttt{max\_overlap} from $10^{-7}$ to 0.5, while setting  \texttt{min\_overlap} = \texttt{max\_overlap}/10. For each overlap setting, we generate 100 distinct data sets and evaluate the average clustering performance of hierarchical clustering, quantified in terms of adjusted mutual information (AMI) and adjusted Rand index (ARI), as in the benchmark of Section \ref{sec:benchmark}. We choose hierarchical clustering because it is computationally efficient and performed well in the benchmark. We repeat this process twice, where in the second run we make clusters non-convex by applying the \texttt{distort} function.

Figure \ref{fig:clustering_difficulty_convex} confirms that clustering difficulty rises with increasing overlap. Figure \ref{fig:clustering_difficulty_nonconvex} shows the same in the case of non-convex clusters, suggesting that applying \texttt{distort} maintains the desired relationship between overlap and clustering difficulty. Additionally, both figures show how our cluster overlap relates to the silhouette score, a popular metric for quantifying clustering difficulty (\citealp{silhouette}; \citealp{hawks}). At a fixed value of \texttt{max\_overlap}, the silhouette score decreases markedly with a rise in dimensionality. This is not an artifact of our overlap measure, since plotting clustering performance vs silhouette score shows a similar dependence on dimensionality (not shown). This makes sense since the silhouette score is based on the difference of Euclidean distances, and distances between points tend to become more similar in high dimensions (\citealp{HighDim}).

\begin{figure}[htbp]
    \centering
    \begin{subfigure}{0.32\textwidth}
        \centering
        \includegraphics[width=\textwidth]{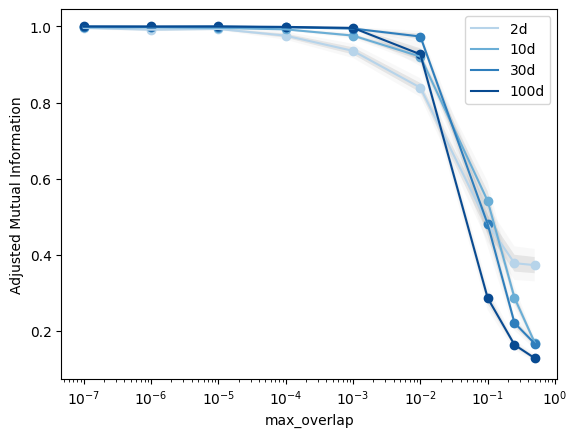}
        \caption{AMI vs Overlap}
    \end{subfigure}
    \begin{subfigure}{0.32\textwidth}
        \centering
        \includegraphics[width=\textwidth]{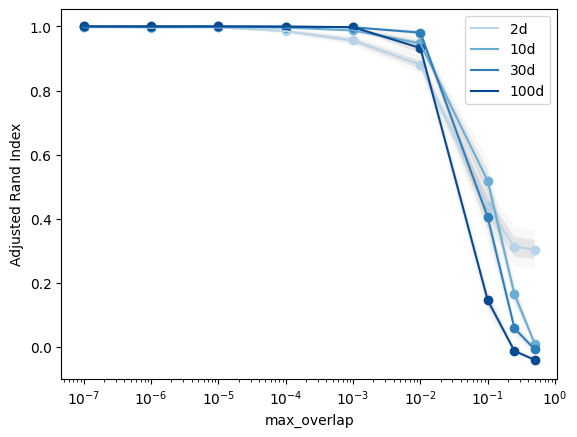}
        \caption{ARI vs Overlap}
    \end{subfigure}
    \begin{subfigure}{0.32\textwidth}
        \centering
        \includegraphics[width=\textwidth]{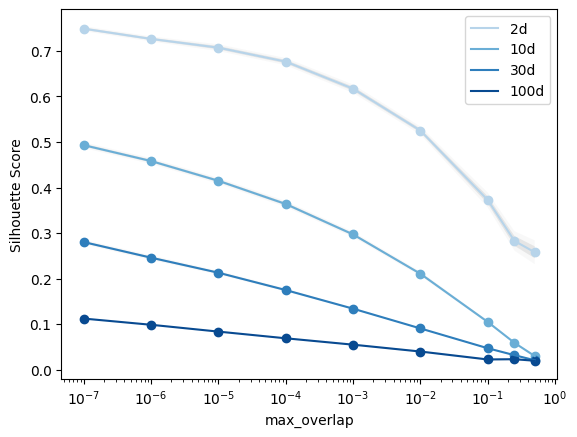}
        \caption{Silhouette vs Overlap}
    \end{subfigure}
    
    \caption{Cluster overlap predicts clustering difficulty for convex clusters. Clustering performance is measured in terms of adjusted mutual information (AMI, left) and adjusted Rand index (ARI, middle). Right: the silhouette score is more sensitive to dimensionality but otherwise aligns well with our cluster overlap.}
    \label{fig:clustering_difficulty_convex}
\end{figure}

\begin{figure}[htbp]
    \centering
    \begin{subfigure}{0.32\textwidth}
        \centering
        \includegraphics[width=\textwidth]{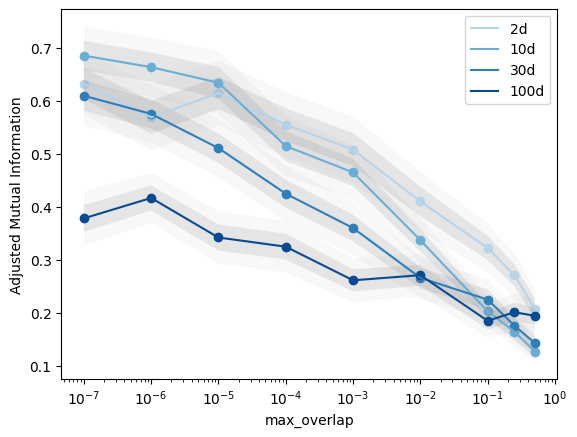}
        \caption{AMI vs Overlap}
    \end{subfigure}
    \begin{subfigure}{0.32\textwidth}
        \centering
        \includegraphics[width=\textwidth]{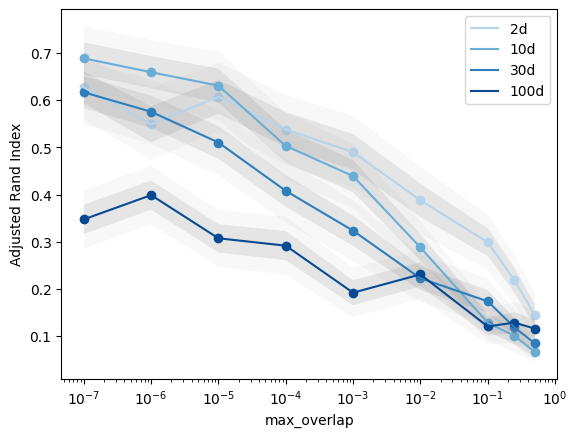}
        \caption{ARI vs Overlap}
    \end{subfigure}
    \begin{subfigure}{0.32\textwidth}
        \centering
        \includegraphics[width=\textwidth]{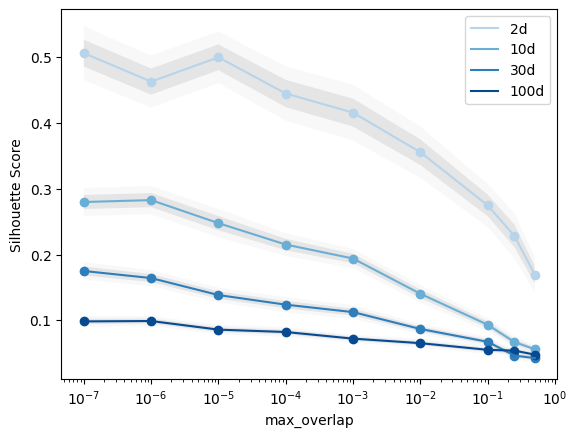}
        \caption{Silhouette vs Overlap}
    \end{subfigure}
    
    \caption{Cluster overlap predicts clustering difficulty for non-convex clusters. Clustering performance is measured in terms of adjusted mutual information (AMI, left) and adjusted Rand index (ARI, middle). Right: the silhouette score is more sensitive to dimensionality but otherwise aligns well with our cluster overlap.}
    \label{fig:clustering_difficulty_nonconvex}
\end{figure}

\subsubsection{Examining the Distribution of Pairwise Overlaps}
\label{sec:overlap_distributions}

Since \texttt{repliclust} controls cluster overlap on the level of entire data sets by setting two global parameters (\texttt{max\_overlap} and \texttt{min\_overlap}), it is worthwhile to investigate the distributions of pairwise overlaps on datasets with multiple clusters. Figure \ref{fig:overlap_distributions} shows the distribution of pairwise overlaps for six data set archetypes, confirming that setting \texttt{max\_overlap} and \texttt{min\_overlap} effectively controls the pairwise overlaps between clusters.

\begin{figure}[h!]
\centering
    \includegraphics[width=\textwidth]{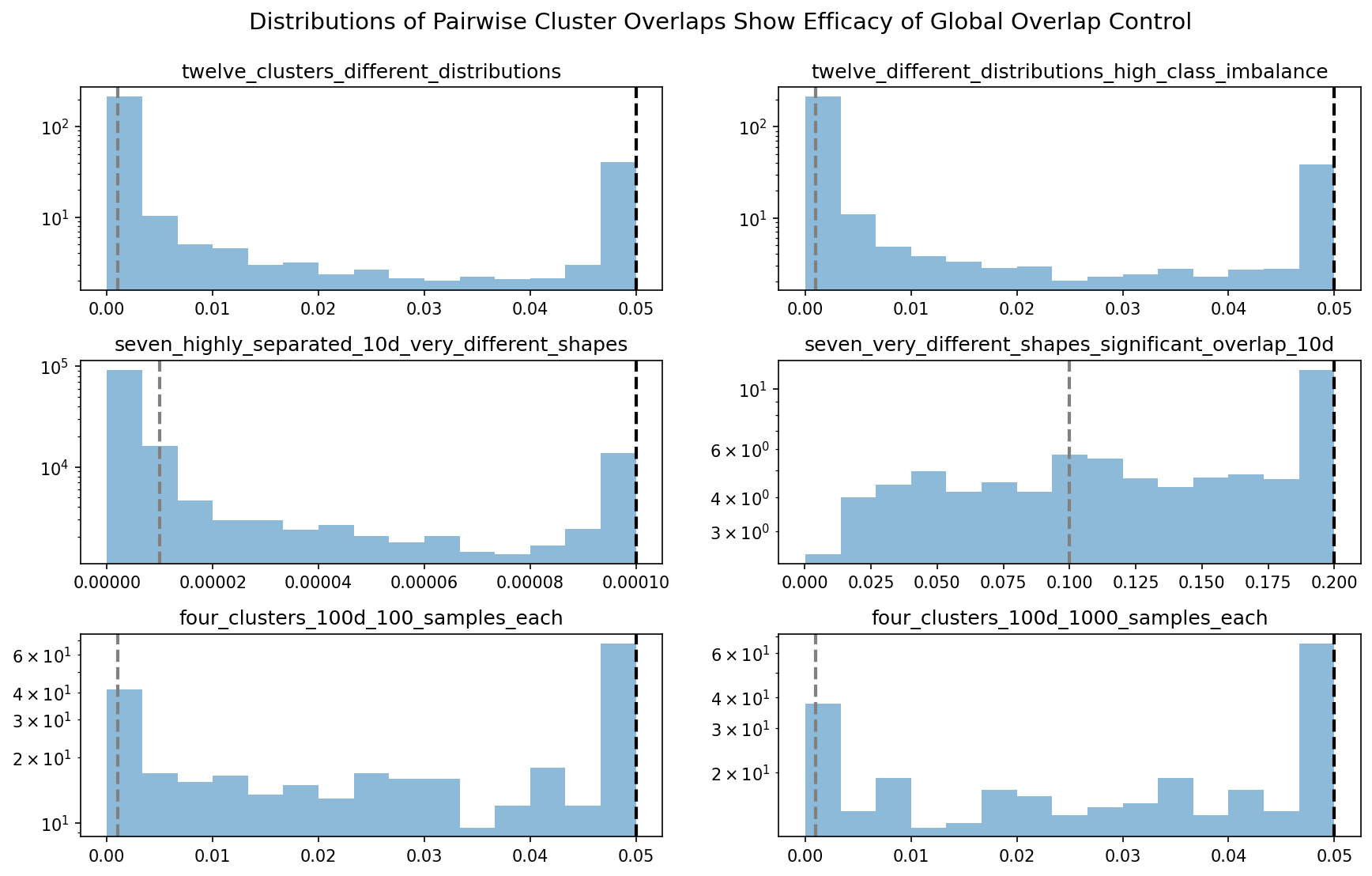}
    \caption{Distributions of pairwise overlaps between clusters reveal that global overlap control works well. The black dashed line indicates the \texttt{max\_overlap} setting, and the gray dashed line indicates \texttt{min\_overlap}. Note that \texttt{min\_overlap} is not a lower bound on pairwise overlap because it only requires that each cluster share the minimum degree of overlap with \textit{at least one} other cluster; overlaps with other clusters may be lower.}
    \label{fig:overlap_distributions}
\end{figure}

\section{Related Work}
\label{sec:relatedwork}

Simulations on synthetic data play an important role in cluster analysis. Accordingly, many synthetic data generators for cluster analysis have been proposed. However, the key idea in this paper is to specify the overall geometric characteristics of synthetic data in a high-level manner via data set archetypes. By contrast, previous data generators have put the onus on the user to \textit{design} the overall geometric structure by \textit{carefully tuning} lower-level properties of individual clusters.

In the following overview, we mainly focus on general-purpose generators. However, the literature has also proposed more specialized solutions to fill specific needs in the community. For example, \cite{subclugen} present a data generator for subspace clustering, \cite{dbscanrevisited} evaluates density-based clustering using a data generation process based on random walks, and \cite{handl05} focus on creating long and thin ellipsoidal clusters in higher dimensions. None of these contributions focus on giving high-level control over the overall geometry of the data sets.

\cite{milligan1985} implements a generator for generating several clusters in up to 8 dimensions. The method enforces an upper bound on cluster overlap by limiting overlap in the first dimension, but does not otherwise provide control over high-level geometric structure. 

\cite{peizaiane06} present a compelling software system for generating two-dimensional clusters, which creates data sets with specified clustering difficulty ``easy'', ``medium'' or ``hard''; ``easy'' data consists only of spherical/convex clusters, whereas ``medium'' and ``hard'' data include curve segments and special shapes like letters of the alphabet. The generator does not offer high-level control over data set geometry except for the difficulty scale and the density of noise points to add to the data. The software comes with an appealing graphical user interface.

The popular scikit-learn library for machine learning (\citealp{scikit-learn}) offers several functions for creating synthetic clusters. Among these, some are aimed at reproducing canonical 2D toy data sets like concentric circles and nested moons (\texttt{make\_moons}, \texttt{make\_circles}), while others focus on sampling multivariate normal clusters. These functions offer some valuable features, such as the ability to create datasets with informative and redundant features, as in the \texttt{make\_classification  function}. However, they do not control overall geometric characteristics of the data.

The data mining framework ELKI (\citealp{elki}) provides a synthetic data generator based on specifying probabilistic models in XML files. This XML-based approach makes it easy to reproduce benchmarks by sharing the underlying XML files. Drawing inspiration from this work, we have implemented an \texttt{Archetype.describe} function allowing users to easily share collections of data set archetypes as JSONL files.

The generators OCLUS (\citealp{oclus}) and GenRandomClust (\citealp{qiujoe06}) focus on providing more sophisticated overlap control compared to previous generators. GenRandomClust extends the generator of \citealp{milligan1985} by managing overlaps between clusters with different ellipsoidal shapes and arbitrary spatial orientations. Similar to our classification error-based notion of cluster overlap, their method finds an optimal separation direction between two clusters. To enforce the desired degree of overlap, the algorithm initially places cluster centers on a scaffold, then scales individual edges of the scaffold up or down to meet the overlap constraint. The method supports making \textit{all} cluster shapes more or less elongated, but does not otherwise provide high-level control over data set geometry. Moreover, the scaling operations undertaken to manage cluster overlaps implicitly sacrifice control over cluster volumes (which, in \texttt{repliclust}, can be managed independently).

OCLUS (\citealp{oclus}) quantifies cluster overlap in terms of the shared density between two clusters. The generator uses analytical formulas for integrals of several interesting probability distributions (including exponential, gamma, and chi-square), thereby effectively managing overlaps between non-normal clusters. As a result, the method is limited to treating all dimensions independently, so that cluster distributions simplify as the products of marginals. The paper contains a valuable discussion of the distinction between marginal and joint cluster overlaps (of which their software supports both).

Like other existing generators, OCLUS does not aspire to helping the user establish the overall geometric characteristics of synthetic data sets. To generate a data set, the user must provide a covariance matrix for each cluster, the desired overlaps between all pairs of clusters, and a design matrix specifying which clusters overlap at all. In sum, generating 10 clusters in 10D requires supplying over 500 numbers, compared to only a handful in \texttt{repliclust} (or none if the user chooses to describe the archetype in English).

MDCGen (\citealp{mdcgen} is a feature-rich generator that supports many desiderata in cluster analysis, such as overlap control, different probability distributions, subspace clusters, and the ability to add noise points. In particular, it is nice to be able to place noise points away from the clusters, which is made possible by the grid-based strategy for placing cluster centers. MDCGen does not target the overall geometric characteristics of synthetic data sets, instead giving users low-level control enabling extensive configurability. For example, managing the overlap between clusters involves setting compactness coefficients, grid granularity, and overall scale, compared to only tweaking $\texttt{max\_overlap}$ in \texttt{repliclust} (\texttt{min\_overlap} may also have to be tweaked but it can usually stay at \texttt{max\_overlap}/10 or a similar value). In the words of the authors, ``to enable covering a broad range of dataset possibilities, the parameters are multiple and some training for tuning the tool is required.''

Finally, the HAWKS generator (\citealp{hawks}) controls cluster overlaps using an evolutionary algorithm that evolves the means and covariance matrices of multivariate normal distributions. The paper applies this framework to create data sets with a user-specified silhouette score representing clustering difficulty (\citealp{silhouette}). In principle, the evolutionary framework can be extended to attain desired high-level geometric characteristics. Of these, the authors consider two examples, cluster overlaps and elongations (the latter relating to our notion of cluster aspect ratio, as listed in Table \ref{tbl:archetype_params}). An interesting aspect of HAWKS is the ability to generate data sets that maximize the performance difference between two clustering algorithms. This feature is especially useful in two dimensions, since we can then visually examine the data sets to better understand when each algorithm succeeds or fails.

\section{Conclusion}
\label{sec:conclusion}

In this paper, we have presented an archetype-based approach to synthetic data generation for cluster analysis. Our method works by summarizing the overall geometric characteristics of a probabilistic mixture model with a few high-level parameters, and sampling mixture models subject to these constraints. To convey the convenience and informativeness of such an archetype-based approach, we have implemented a natural language interface allowing users to create archetypes purely from verbal descriptions in English. Thus, our software \texttt{repliclust} makes it possible to run an entire benchmark by describing the desired evaluation scenarios in English.

Although our data generator relies on creating a skeleton of convex, ellipsoidal clusters, we have implemented ways to make the cluster shapes more irregular and complex. The first method passes convex clusters through a randomly initialized neural network, making their shapes non-convex and irregular. The second method creates directional datasets by wrapping $p$-dimensional convex clusters around the $(p+1)$-dimensional sphere.

In future work, we are most interested in learning data set archetypes that mimic the geometric characteristics of real data. In application domains where multivariate Gaussians provide a good model for the data, it would suffice to fit a Gaussian mixture model, empirically measure its high-level geometric parameters (as listed in Table \ref{tbl:archetype_params}), then create a synthetic data set archetype with these parameters. For application domains with non-convex clusters, it would be interesting to implement a generalized version of this approach. We can imagine training an auto-encoder type neural network that maps non-convex clusters into a hidden space where they become multivariate Gaussian, then undoes the transformation. Defining an archetype based on the high-level geometric characteristics in the hidden space would presumably allow us to sample irregularly shaped clusters that look similar to those found in real data.

\section*{Acknowledgements}
We thank the anonymous referees for their feedback, which has significantly improved the paper. In addition, MJZ would like to acknowledge support from Matt Thomson and members of the Thomson Lab at Caltech. Finally, MJZ thanks Dante Roy Calapate for providing his computer monitor during the early stages of this project.

\section*{Compliance with Ethical Standards}

The project did not involve any studies with human or animal participants.

\section*{Conflicts of Interest}

The authors report no conflicts of interest.

\section*{Funding}

This project is not tied to a specific grant.

\section*{Data Availability}
Data and code employed in writing the paper is available on GitHub in the repository \texttt{github.com/mzelling/repliclust-analysis}.

\vskip 0.2in

\nocite{kmeans}
\nocite{matplotlib}
\nocite{scipy}
\nocite{numpy}
\nocite{dbscanrevisited}
\nocite{subclugen}
\nocite{lecunbackprop}
\nocite{layernorm}

\bibliography{main}

\appendix
\setcounter{theorem}{2}

\newpage

\section*{Appendix A}

We give more detail on how \texttt{repliclust} manages various geometric attributes using max-min parameters. Table \ref{tbl:maxminattributes} lists all geometric attributes managed with max-min sampling and names the corresponding parameters in \texttt{repliclust}. The \textit{reference value} for each geometric parameter serves as a location constraint, while the \textit{max-min ratio} determines the spread. A \textit{constraint} ensures that the distribution of geometric parameters within a data set is similar across data sets drawn from the same archetype.

\begin{table}[h!]
	\caption{Summary of geometric attributes managed with max-min sampling. The second and third columns indicate whether each max-min ratio or reference value is inferred, or specified by the user as a \texttt{parameter}. The fourth column gives the location constraint used during max-min sampling. The \textit{group size} of a cluster is the number of data points in it; the \textit{aspect ratio} is the ratio of the lengths of the longest cluster axis to the shortest. For cluster volumes, we specify the reference value and max-min ratio in terms of \textit{radius} (\texttt{dim}-th root of volume) since volumes grow rapidly in high dimensions.}
  	\label{tbl:maxminattributes}
  	\centering
  	\centerline{
\resizebox{\textwidth}{!}{
    \begin{tabular}{ | c | c | c | c |}
      \hline
      \textbf{Geometric Attribute} & \textbf{Max-Min Ratio} & \textbf{Reference Value} & \textbf{Constraint} \\
      \hline
      cluster volumes & \makecell{\texttt{radius\_maxmin}} & \makecell{\texttt{scale}} & \makecell{cluster volumes average \\to reference volume}\\
      \hline
      group sizes & \texttt{imbalance\_ratio} & \makecell{average group size} & \makecell{group sizes sum to \\ number of samples}\\
      \hline
      cluster aspect ratios & {\texttt{aspect\_maxmin}} & \makecell{\texttt{aspect\_ref}} &  \makecell{geometric mean of aspect \\ ratios equals reference}\\
      \hline
      \makecell{cluster axis lengths} & \makecell{aspect ratio \\ of the cluster} & \makecell{\texttt{dim}-th root of \\ cluster volume} &  \makecell{geometric mean of lengths \\ equals reference length}\\
      \hline
    \end{tabular}
    }
    }
\end{table}

To enforce the attribute-specific constraints, \texttt{repliclust} samples new values of a geometric parameter in pairs. The first value is randomly drawn from a triangular distribution whose mode and endpoints are determined from the typical value and max-min ratio. The second value is then deterministically computed to maintain the constraint. For reference, see the \texttt{sample\_with\_maxmin} function in the \texttt{maxmin.utils} module (Version 1.0.0 of \texttt{repliclust}).

\section*{Appendix B}

Table \ref{tbl:mixturemodel} lists the formal attributes of a mixture model in \texttt{repliclust}. A data set archetype provides a way to randomly sample mixture models with similar overall geometric characteristics. Thus, an archetype implicitly defines a probability distribution over the attributes in Table \ref{tbl:mixturemodel}.

\begin{table}[h!]
	\caption{Formal attributes of a mixture model in \texttt{repliclust}.}
  	\label{tbl:mixturemodel}
  	\centering
  	\centerline{
  	\resizebox{\textwidth}{!}{
    \begin{tabular}{ | c | c | c |}
      \hline
      \textbf{Attribute} & \textbf{Meaning} & \textbf{Mathematical Definition} \\
      \hline
      \makecell{cluster centers} & the positions of cluster centers in space & \makecell{$\boldsymbol{\mu}_1, \boldsymbol{\mu}_2, ..., \boldsymbol{\mu}_k \in \mathbb{R}^p$} \\
      \hline
      \makecell{principal axis orientations} & \makecell{the spatial orientation of each cluster's \\ ellipsoidal shape (different for each cluster) } & \makecell{orthonormal matrices \\ $\mathbf{U}_1, \mathbf{U}_2, ..., \mathbf{U}_k \in \mathbb{R}^{p \times p}$ } \\
      \hline
      \makecell{principal axis lengths} & \makecell{the lengths of each cluster's principal axes \\ (axes have different lengths between and within clusters)} & \makecell{$\boldsymbol{\sigma}_1, \boldsymbol{\sigma}_2, ..., \boldsymbol{\sigma}_k \in {(\mathbb{R}^{>0}})^p$}\\
      \hline
      \makecell{cluster distributions} & \makecell{multivariate probability distributions for \\ generating data (different for each cluster)} & \makecell{distributions $\mathbb{P}_1, \mathbb{P}_2, ..., \mathbb{P}_k$}\\
      \hline
    \end{tabular}
    }
    }
\end{table}

\section*{Appendix C}

We prove Theorem \ref{thm:ldaoverlap} of Section \ref{sec:overlap}. Additionally, we provide an analogous result for the simpler ``center-to-center'' approximation of cluster overlap.
\begingroup
\addtocounter{theorem}{-2}

\begin{theorem}[LDA-Based Cluster Overlap]
For two multivariate normal clusters with means $\boldsymbol{\mu}_1 \neq \boldsymbol{\mu}_2$ and covariance matrices $\boldsymbol{\Sigma}_1, \boldsymbol{\Sigma}_2$, the approximate cluster overlap $\alpha_\text{\tiny{LDA}}$ based on the linear separator $\boldsymbol{a}_\text{\tiny{LDA}} = (\frac{\boldsymbol{\Sigma}_1 + \boldsymbol{\Sigma}_2}{2})^{-1}(\boldsymbol{\mu}_2 - \boldsymbol{\mu}_1)$ is 
\begin{equation}
\label{eq:thmresult_lda}
\alpha_\text{\tiny{LDA}} = 2\big(1 - \Phi \Big( \frac{\boldsymbol{a}_\text{\tiny{LDA}}^{\top} (\boldsymbol{\mu}_2 - \boldsymbol{\mu}_1)}{\sqrt{\boldsymbol{a}_\text{\tiny{LDA}}^{\top} \boldsymbol{\Sigma}_1 \boldsymbol{a}_\text{\tiny{LDA}}} + \sqrt{\boldsymbol{a}_\text{\tiny{LDA}}^{\top}\boldsymbol{\Sigma}_2\boldsymbol{a}_\text{\tiny{LDA}}}} \Big) \big),
\end{equation}
where $\Phi(z)$ is the cumulative distribution function of the standard normal distribution. 
Moreover, if $\boldsymbol{\Sigma}_1 = \lambda \boldsymbol{\Sigma}_2$ for some $\lambda$ then $\alpha_\text{\tiny{LDA}}$ equals the exact cluster overlap $\alpha$.
\end{theorem}

\begin{proof}
Let $\boldsymbol{a}_\text{\tiny{LDA}}$ be the classification axis. Minimax optimality requires that the cluster-specific misclassification probabilities are equal. Since $\boldsymbol{a}_\text{\tiny{LDA}}$ is the classification axis, these probabilities correspond to the tails of the marginal distributions along $\boldsymbol{a}_\text{\tiny{LDA}}$. Specifically, let
\begin{equation}
\sigma_1 = \sqrt{\boldsymbol{a}_\text{\tiny{LDA}}^\top \Sigma_1 \boldsymbol{a}_\text{\tiny{LDA}}}
\end{equation}
be the standard deviation of cluster 1's marginal distribution along $\boldsymbol{a}_\text{\tiny{LDA}}$, where $\Sigma_1$ is the cluster's covariance matrix; $\sigma_2$ is defined analogously. If $\boldsymbol{a}_\text{\tiny{LDA}}$ is oriented to point from cluster 1 to cluster 2, then the $1-\alpha/2$ quantile of cluster 1's marginal distribution meets the $\alpha/2$ quantile of cluster 2's marginal distribution at the decision boundary, where $\alpha$ is the unknown cluster overlap. This intersection implies
\begin{equation}
    \boldsymbol{\mu}_1^\top \boldsymbol{a}_\text{\tiny{LDA}}  + q_{1-\alpha/2}\sigma_1 =  \boldsymbol{\mu}_2^\top \boldsymbol{a}_\text{\tiny{LDA}}  + q_{\alpha/2}\sigma_2, 
\end{equation}
where $q_{\xi}$ is the $\xi$-quantile of the standard normal distribution.
Rearranging this equation, and using $q_{\alpha/2} = -q_{1-\alpha/2}$ and $\Phi(q_\xi) = \xi$, gives (\ref{eq:thmresult_lda}).

Next, suppose that $\boldsymbol{\Sigma}_1 = \lambda \boldsymbol{\Sigma}_2$ for some $\lambda$. In this case, maximum likelihood classification results in a linear decision boundary that coincides with the LDA solution. Hence, the minimax-optimal linear classifier uses the LDA-based classification axis $\boldsymbol{a}_\text{\tiny{LDA}}$.
\end{proof}

\begin{theorem}[Center-to-Center Cluster Overlap]
For two multivariate normal clusters with means $\boldsymbol{\mu}_1 \neq \boldsymbol{\mu}_2$ and covariance matrices $\boldsymbol{\Sigma}_1, \boldsymbol{\Sigma}_2$, the center-to-center cluster overlap $\alpha_\text{\tiny{C2C}}$, based on a classification boundary perpendicular to the line connecting the cluster centers, is 

\begin{equation}
\alpha_\text{\tiny{C2C}} = 2\big(1 - \Phi \Big( \frac{\boldsymbol{\delta}^{\top} \boldsymbol{\delta}} {\sqrt{\boldsymbol{\delta}^{\top} \boldsymbol{\Sigma}_1 \boldsymbol{\delta}} + \sqrt{\boldsymbol{\delta}^{\top}\boldsymbol{\Sigma}_2\boldsymbol{\delta}}} \Big) \big),
\end{equation}
where $\boldsymbol{\delta} := \boldsymbol{\mu}_2 - \boldsymbol{\mu}_1$ is the difference between cluster centers and $\Phi(z)$ is the cumulative distribution function of the standard normal distribution. 

Moreover, if the covariance matrices $\boldsymbol{\Sigma}_1$ and $\boldsymbol{\Sigma}_2$ are both multiples of the identity matrix, then $\alpha_\text{\tiny{C2C}}$ equals the exact cluster overlap $\alpha$.
\end{theorem}

\begin{proof}
The proof proceeds along the same lines as the proof of Theorem \ref{thm:ldaoverlap}, except that the classification axis is $\boldsymbol{\mu}_2 - \boldsymbol{\mu}_1$. If both covariance matrices are multiples of the identity matrix, $\boldsymbol{\mu}_2 - \boldsymbol{\mu}_1$ is a scalar multiple of the LDA-based classification axis $\boldsymbol{a}_\text{\tiny{LDA}}$. Hence, the second part of Theorem \ref{thm:ldaoverlap} kicks in to establish equality between $\alpha_\text{\tiny{C2C}}$ and the exact overlap.
\end{proof}
\endgroup

\section*{Appendix D}

Figure \ref{fig:viz_nonnormal_overlap} visualizes two-dimensional data sets created from the same archetype but with different radial probability distributions. The results suggest that pegging the 68.15\% quantile of the radial distribution at unity leads to satisfactory overlap control for distributions with infinite support. However, for distributions with bounded support (such as the beta distribution), this approach leads to greater separation between the clusters, as shown in the rightmost column of Figure \ref{fig:viz_nonnormal_overlap}.

\begin{figure}[h!]
\centering
    \includegraphics[width=\textwidth]{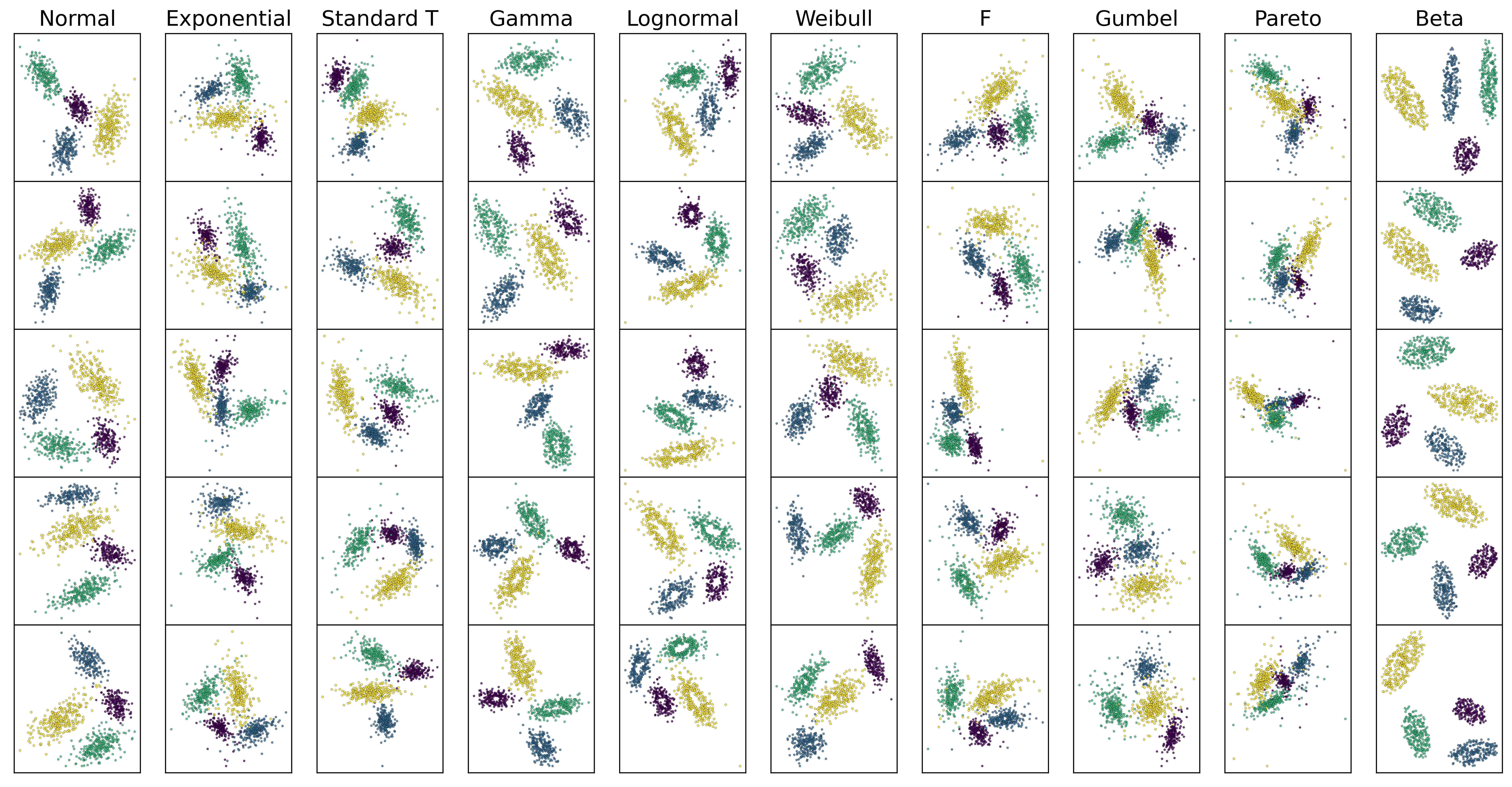}
    \caption{Overlap control works well for non-normal probability distributions with infinite support. The data sets shown are generated from the same archetype in 2D (with overlap at around 1\%), except that we change the probability distribution in each column. The only distribution that violates $1\%$ overlap is the beta distribution, which unfortunately generalizes to other distributions with bounded support. Note that the heavy-tailed distributions (Pareto, F, ...) appear smaller on scatter plots because they give rise to outliers.}
    \label{fig:viz_nonnormal_overlap}
\end{figure}

\section*{Appendix E}

The default architecture for the neural network used in the \texttt{distort} function of Section \ref{sec:realism} is a feed-forward network consisting (in order) of a linear embedding, 16 repeated feed-forward blocks (each consisting of a fully connected linear layer followed by layer normalization and Tanh activation), and a linear projection. Importantly, we tie the weights of the embedding and projection layers, so that the projection weights are the transpose of the embedding weights (\citealp{weighttying}). The default hidden dimensionality is 128, so that the embedding layer maps a data set archetype's dimension to 128. The internal feed-forward blocks preserve this hidden dimension, and the final projection layer maps it back to the archetype's dimension.

\section*{Appendix F}

Below we list the prompt templates (including few-shot examples) used in Version 1.0.0 of \texttt{repliclust}. Up-to-date versions are available in the code base (see \texttt{repliclust.org}).

\vspace{5mm}

\noindent \textbf{1. Prompt template mapping archetype description to high-level geometric parameters:}

\begin{lstlisting}[basicstyle=\ttfamily\footnotesize, breaklines=true]
Your task is to turn a verbal description of a data set archetype from Repliclust into a precise JSON that specifies which parameter settings to use to create the desired data set archetype in Repliclust. These are the allowed parameters:

n_clusters: int >= 1, the number of clusters to generate
dim: int >= 2, the dimensionality of the data
n_samples: int >= 1, the number of data samples to generate
aspect_ref: float >= 1, the eccentricity of a typical cluster (how oblong vs spherical it is)
aspect_maxmin: float >= 1, how much the eccentricity varies across clusters in a data set
radius_maxmin: float >= 1, how much cluster radius (and thereby cluster volume) varies across the clusters
max_overlap: float > 0, the maximum allowed overlap between any pair of clusters (0.1-0.2 is significant overlap, 0.01-0.05 is little overlap, 0.001 is very little overlap, and 0.0001 and lower is well-separated)
min_overlap: float > 0, the minimum amount of overlap each cluster should have with some other cluster, preventing a cluster from being too far away from all other clusters
imbalance_ratio: float >= 1, specifies how imbalanced the number of data points per cluster is
distributions: list[str], determines the probability distributions to use for the clusters; the available distributions are 'normal', 'standard_t', 'exponential', 'beta', 'uniform', 'chisquare', 'gumbel', 'weibull', 'gamma', 'pareto', 'f', and 'lognormal'

IMPORTANT NOTES:
Any words like "separated", "far away", "close together", or "overlapping" refer to the overlap between clusters. Far apart means that max_overlap is 1e-4 or less
Always make min_overlap smaller than max_overlap, usually ten times smaller!
ONLY include the Pareto ('pareto') distribution if the user specifically asks for heavy tails!

EXAMPLES:

Description: five oblong clusters in two dimensions
Archetype JSON: {
  "n_clusters": 5,
  "dim": 2,
  "n_samples": 500,
  "aspect_ref": 3,
  "aspect_maxmin": 1.5,
}

Description: three spherical clusters with significant overlap in two dimensions
Archetype JSON: {
  "n_clusters": 3,
  "dim": 2,
  "n_samples": 300,
  "max_overlap": 0.2,
  "min_overlap": 0.1,
  "aspect_ref": 1.0,
  "aspect_maxmin": 1.0
}

Description: eight spherical clusters of different sizes with significant overlap in two dimensions
Archetype JSON: {
  "n_clusters": 8,
  "dim": 2,
  "n_samples": 800,
  "max_overlap": 0.25,
  "min_overlap": 0.1,
  "aspect_ref": 1.0,
  "aspect_maxmin": 1.0,
  "radius_maxmin": 2.0
}

Description: ten clusters which are all highly oblong (and equally so) but of very different sizes, with moderate overlap
Archetype JSON: {
  "n_clusters": 10,
  "n_samples": 1000,
  "aspect_ref": 5,
  "aspect_maxmin": 1.0,
  "max_overlap": 0.10,
  "min_overlap": 0.05,
  "radius_maxmin": 4.0
}

Description: five clusters with significant class imbalance
Archetype JSON: {
  "n_clusters": 5,
  "n_samples": 500,
  "imbalance_ratio": 5,
  "aspect_ref": 1.5,
  "aspect_maxmin": 1.5
}

Description: five clusters with perfect class balance
Archetype JSON: {
  "n_clusters": 5,
  "n_samples": 500,
  "imbalance_ratio": 1.0,
  "aspect_ref": 1.4,
  "aspect_maxmin": 1.6
}

Description: eight clusters of which 70% are exponentially distributed and 30% are normally distributed
Archetype JSON: {
  "n_clusters": 8,
  "n_samples": 800,
  "aspect_ref": 1.7,
  "aspect_maxmin": 1.5,
  "distributions": ["exponential", "normal"],
  "distribution_proportions": [0.7, 0.3],
}

Description: eight clusters with 1000 total samples of which half are exponentially distributed and half are normally distributed
Archetype JSON: {
  "n_clusters": 8,
  "n_samples": 1000,
  "aspect_ref": 1.7,
  "aspect_maxmin": 1.5,
  "distributions": ["exponential", "normal"],
  "distribution_proportions": [0.5, 0.5]
}

Description: two clusters of different sizes in 10 dimensions that are well-separated
Archetype JSON: {
  "n_clusters": 2,
  "dim": 10,
  "n_samples": 200,
  "aspect_ref": 2
  "aspect_maxmin": 2,
  "radius_maxmin": 4.0,
  "max_overlap": 0.001,
  "min_overlap": 0.0001
}

Description: very oblong clusters that overlap heavily
Archetype JSON: {
  "n_clusters": 6,
  "n_samples": 600,
  "aspect_ref": 7,
  "aspect_maxmin": 1.4,
  "max_overlap": 0.4,
  "min_overlap": 0.3
}

Description: highly separated and very oblong clusters
Archetype JSON: {
  "n_clusters": 4,
  "n_samples": 400,
  "aspect_ref": 6,
  "aspect_maxmin": 1.6,
  "max_overlap": 1e-4,
  "min_overlap": 1e-5
}

Description: ten clusters with very different shapes
Archetype JSON: {
  "n_clusters": 10,
  "n_samples": 1000,
  "aspect_ref": 1.5,
  "aspect_maxmin": 3.0,
  "radius_maxmin": 3.0
}

Description: twelve well-separated clusters with very different shapes
Archetype JSON: {
  "n_clusters": 12,
  "n_samples": 1200,
  "aspect_ref": 1.5,
  "aspect_maxmin": 5.0,
 "radius_maxmin": 5.0, 
 "max_overlap": 1e-4,
  "min_overlap": 1e-5
}}

Description: twelve highly separated Gaussian clusters with very different shapes
Archetype JSON: {
  "n_clusters": 12,
  "n_samples": 1200,
  "aspect_ref": 1.5,
  "aspect_maxmin": 5.0,
 "radius_maxmin": 5.0, 
 "max_overlap": 1e-4,
  "min_overlap": 1e-5
 "distributions": ["normal"]}}

Description: five heavy-tailed clusters
Archetype JSON: {
  "n_clusters": 5,
  "n_samples": 500,
  "aspect_ref": 1.5,
 "distributions": ["standard_t", "lognormal", "pareto"]}}

Description: clusters with holes
Archetype JSON: {"distributions": ["f"]}

Description: clusters from a variety of distributions
Archetype JSON: {"distributions": ["normal", "exponential", "gamma", "weibull", "lognormal"]}

Description: clusters from all different distributions
Archetype JSON: {"distributions": ['normal', 'standard_t', 'exponential', 'beta', 'uniform', 'chisquare', 'gumbel', 'weibull', 'gamma', 'f', and 'lognormal']}

Description: clusters from different distributions
Archetype JSON: {"distributions": ['normal', 'exponential', 'beta', 'uniform', 'chisquare', 'gumbel', 'weibull', 'gamma', 'f', and 'lognormal']}

Description: highly separated clusters from all different distributions but no heavy tails
Archetype JSON: {"max_overlap": 1e-4,
  "min_overlap": 1e-5,
 "distributions": ['normal', 'exponential', 'beta', 'uniform', 'chisquare', 'gumbel', 'weibull', 'gamma', 'f', and 'lognormal']}

Description: seven clusters with uniform distribution with light overlap
Archetype JSON: { "max_overlap": 0.025, 
"min_overlap": 0.0025,
 "distributions": ["uniform"]}

Description: clusters with bounded support
Archetype JSON: {"distributions": ["beta", "uniform"]}

Description: {description}
Archetype JSON:
\end{lstlisting}

\vspace{5mm}

\noindent \textbf{2. Prompt template mapping archetype description to a descriptive identifier:}

\begin{lstlisting}[basicstyle=\ttfamily\footnotesize, breaklines=true]
Your task is to turn a description of a data set archetype into an identifier for
the archetype. The identifier should be short yet descriptive, and not contain any whitespace
(but underscores are OK). IMPORTANT: the identifier should be a valid Python variable name.
Specifically, it may NOT start with a number, nor contain any special character except for
underscores.

EXAMPLES:

Description: five oblong clusters in two dimensions
Archetype identifier: five_oblong_2d

Description: three spherical clusters with significant overlap in two dimensions
Archetype identifier: three_spherical_significant_overlap_2d

Description: eight spherical clusters of different sizes with significant overlap in two dimensions
Archetype identifier: eight_spherical_different_sizes_significant_overlap_2d

Description: ten clusters which are all highly oblong (and equally so) but of very different sizes, with moderate overlap
Archetype identifier: ten_highly_oblong_very_different_shapes_moderate_overlap

Description: five clusters with significant class imbalance
Archetype identifier: five_significant_class_imbalance

Description: five clusters with perfect class balance
Archetype identifier: five_perfect_class_balance

Description: eight clusters of which 70% are exponentially distributed and 30% are normally distributed
Archetype identifier: eight_exp_and_norm

Description: eight clusters with 1000 total samples of which half are exponentially distributed and half are normally distributed
Archetype identifier: eight_exp_and_norm_1000_samples

Description: two clusters of different sizes in 10 dimensions that are well-separated
Archetype identifier: two_different_sizes_well_separated_10d

Description: very oblong clusters that overlap heavily
Archetype identifier: very_oblong_heavy_overlap

Description: ten clusters with very different shapes
Archetype identifier: ten_very_different_shapes

Description: {description}
Archetype identifier:
\end{lstlisting}

\end{document}